\documentclass[a4paper, 10pt]{article}
\usepackage{packageList}
\usepackage{macros}

\usepackage{cleveref}

\newenvironment{delayedproof}[1]
 {\begin{proof}[\raisedtarget{#1}Proof of \Cref{#1}]}
 {\end{proof}}
\newcommand{\raisedtarget}[1]{%
  \raisebox{\fontcharht\font`P}[0pt][0pt]{\hypertarget{#1}{}}%
}
\newcommand{\proofref}[1]{\hyperlink{#1}{proof of \Cref{#1}}}

\newtheorem{prop}[theorem]{Proposition}

\graphicspath{{./Figures/}}

\newcommand{\relu}{\text{ReLU}}

\newtheorem{assumption}{Assumption}

\pgfplotsset{compat=1.11}

\title{Analysis of Structured Deep Kernel Networks}
\author[1,2]{Tizian Wenzel \thanks{tizian.wenzel@mathematik.uni-stuttgart.de, corresponding author}}
\author[3]{Gabriele Santin \thanks{gabriele.santin@unive.it, \href{http://orcid.org/0000-0001-6959-1070}{orcid.org/0000-0001-6959-1070}}}

\author[1]{Bernard Haasdonk \thanks{haasdonk@mathematik.uni-stuttgart.de}}
\affil[1]{Institute for Applied Analysis and Numerical Simulation, University of Stuttgart, Germany}
\affil[2]{Deparment of Mathematics, Ludwig-Maximilians-Universität München, Germany}
\affil[3]{Ca' Foscari University of Venice (Venice, Italy)}

\begin{document}

\maketitle 
\begin{abstract}

In this paper, we leverage a recent deep kernel representer theorem to connect kernel based learning and (deep) neural networks in order to understand their interplay. 
In particular, we show that the use of special types of kernels yields models reminiscent of neural networks that are founded in the same theoretical framework of classical kernel methods, while benefiting from the computational advantages of deep neural networks. 
Especially the introduced Structured Deep Kernel Networks (SDKNs) can be viewed as neural networks (NNs) with optimizable activation functions obeying a representer theorem. 
This link allows us to analyze also NNs within the framework of kernel networks.

We prove analytic properties of the SDKNs which show their universal approximation properties in three different asymptotic regimes of unbounded number of 
centers, width and depth. 
Especially in the case of unbounded depth, more accurate constructions can be achieved using fewer layers compared to corresponding constructions for ReLU 
neural networks.
This is made possible by leveraging properties of kernel approximation.

\end{abstract}

\section{Introduction}\label{sec:introduction}

Kernel-based methods \cite{Wendland2005} rely on positive definite kernels and their associated Reproducing Kernel Hilbert Space (RKHS), 
offering a unified framework for analysis via the representer theorem \cite{Wahba1970, Schoelkopf2001v}. 
These methods apply simple linear algorithms to input data mapped through a nonlinear 
feature map to a possibly infinite-dimensional feature space. This feature map is 
either implicitly defined by the kernel or manually designed to capture similarities between relevant data structures,
including non-Euclidean structured data like graphs. \\
Another approach which gained more popularity in the recent decade due to increased computational power and unprecendented availability of large data sets is given by deep neural networks \cite{Goodfellow2016}. 
These methods manage to reach an excellent accuracy on very high-dimensional problems, but are based on completely different premises: Based on a multilayer setup,
composed of linear mappings and nonlinear activation functions, they also learn a feature representation of the data themselves instead of using a fixed feature 
map. 

It is thus particularly natural and promising to investigate multilayer kernel approximants obtained by composition of several kernel-based functions. 
Although this approach has been already pursued in some directions \cite{NIPS2009_3628,damianou2013deep}, 
a recent result has laid the foundations of a new analysis through the establishment of a representer theorem for multilayer kernels \cite{Bohn2017}, which 
reduces an infinite dimensional problem to a finite one. 
Still, the actual solution of the corresponding optimization problem presents several challenges both on a computational and on a theoretical level. 
First, despite one can move from infinite to finite-dimensional optimization via the representer theorem, the problem is non-convex and still very hard to solve 
numerically.
Second, the rich theoretical foundation of the shallow (single-layer) case can not be directly transferred to this situation. \\
Our approach introduces three key innovations:
First we introduce the Structured Deep Kernel Network (SDKN) framework.
Second, we provide a theoretical analysis demonstrating that basic functions, such as polynomials, can be approximated to arbitrary accuracy by simple SDKNs.
Finally, these results are combined to establish efficient universal approximation theorems.

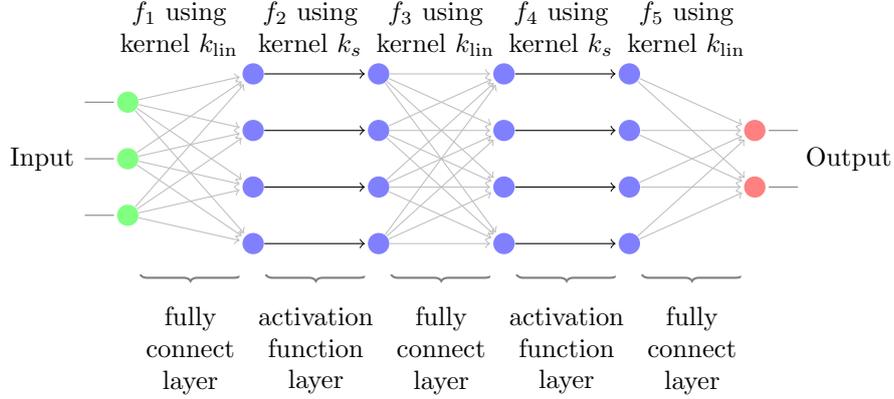
\begin{figure}[H]
\def\layersep{2.2cm}
\centering %
\begin{tikzpicture}[scale=0.75, shorten >=1pt,->,draw=black!50, node distance=\layersep]
    \tikzstyle{every pin edge}=[<-,shorten <=1pt]
    \tikzstyle{neuron}=[circle,fill=black!25,minimum size=8pt,inner sep=0pt]
    \tikzstyle{input neuron}=[neuron, fill=green!50];
    \tikzstyle{output neuron}=[neuron, fill=red!50];
    \tikzstyle{hidden neuron}=[neuron, fill=blue!50];
    \tikzstyle{annot} = [text width=5em, text centered]

	\node[annot] () at (-1.5, -2.5) {Input};
	\node[annot] () at (5.75*\layersep, -2.5) {Output};

    \foreach \name / \y in {1,2,3}
        \node[input neuron,pin={[pin edge={-}]left:}] (I-\name) at (0,-.5-\y) {};

    \foreach \name / \y in {1,...,4}
        \node[hidden neuron] (H1-\name) at (\layersep,-\y) {};
        
    \foreach \name / \y in {1,...,4}
        \node[hidden neuron] (H2-\name) at (2*\layersep,-\y) {};
        
    \foreach \name / \y in {1,...,4}
        \node[hidden neuron] (H3-\name) at (3*\layersep,-\y) {};
        
    \foreach \name / \y in {1,...,4}
        \node[hidden neuron] (H4-\name) at (4*\layersep,-\y) {};

    \foreach \name / \y in {1,2}
        \node[output neuron,pin={[pin edge={-}]right:}] (O-\name) at (5*\layersep,-1-\y) {};    

    \foreach \source in {1,...,3}
	    \foreach \dest in {1,...,4}
	        \path [lightgray] (I-\source) edge (H1-\dest);    
    
    \foreach \source in {1,...,4}
        \foreach \dest in {1,...,4}
            \path [lightgray] (H2-\source) edge (H3-\dest);
            
	\foreach \source in {1,...,4}
		\foreach \dest in {1,2}
	        \path [lightgray] (H4-\source) edge (O-\dest);   

	\foreach \node in {1,...,4}
		\path [black] (H1-\node) edge (H2-\node);
		
	\foreach \node in {1,...,4}
		\path [black] (H3-\node) edge (H4-\node);

	\node[annot] (hl1) at (.9, -.2) {$f_1$ using kernel $k_\text{lin}$};
    \node[annot] (hl2) at (3.2, -.2) {$f_2$ using kernel $k_s$};
    \node[annot] (hl3) at (5.4, -.2) {$f_3$ using kernel $k_\text{lin}$};
	\node[annot] (hl4) at (7.6, -.2) {$f_4$ using kernel $k_s$};
	\node[annot] (hl5) at (9.8, -.2) {$f_5$ using kernel $k_\text{lin}$};

	\draw[decorate, thick, -, decoration={brace,mirror}, yshift=3ex]  (.1*\layersep, -5) -- node[below=2ex, align=center, midway] {fully \\ connect \\ layer}  (.9*\layersep, -5);
	\draw[decorate, thick, -, decoration={brace,mirror}, yshift=3ex]  (1.1*\layersep, -5) -- node[below=2ex, align=center, midway] {activation \\ function \\ layer}  (1.9*\layersep, -5);
	\draw[decorate, thick, -, decoration={brace,mirror}, yshift=3ex]  (2.1*\layersep, -5) -- node[below=2ex, align=center, midway] {fully \\ connect \\ layer}  (2.9*\layersep, -5);	
	\draw[decorate, thick, -, decoration={brace,mirror}, yshift=3ex]  (3.1*\layersep, -5) -- node[below=2ex, align=center, midway] {activation \\ function \\ layer}  (3.9*\layersep, -5);	
	\draw[decorate, thick, -, decoration={brace,mirror}, yshift=3ex]  (4.1*\layersep, -5) -- node[below=2ex, align=center, midway] {fully \\ connect \\ layer}  (4.9*\layersep, -5);	

\end{tikzpicture}
\caption{Visualization of the overall Structured Deep Kernel Network (SDKN) architecture (see Section~\ref{sec:sdkn} for a precise definition). Throughout the 
figures, gray arrows refer to layers using the linear kernel, while black arrows refer to layers using the single dimensional kernel. The braces below the 
layers indicate the similarities to neural networks. \newline
The illustrated SDKN has a depth of $L=2$, a maximal dimension (width) of $w = 4$ and the dimensions $d_0 = 2, d_1 = d_2 = d_3 = d_4 = 4, d_5 = 2$ (see notation introduced in Section \ref{sec:analysis}).}
\label{fig:visualization_NN_DK}
\end{figure}

In this view, we start by analyzing possible ways of using the deep representer theorem and highlight some limitations of a straightforward approach (Section \ref{subsec:rbf_approach}). 
We then use our insights to introduce a suitable architecture (i.e., a specification of the choice of the kernels and of their composition into multiple 
layers--Figure~\ref{fig:visualization_NN_DK}) that makes the solution of the optimization problem possible with efficient methods. 
More precisely, we construct two basic layers, one that is using the linear kernel, and one that is built by means of a nonlinear kernel applied componentwise (Section \ref{subsec:setup_sdkn}). 
We use these layers as building blocks that can be stacked into an SDKN, whose architecture is reminiscent of feedforward neural networks. 
We comment on the optimality and the relation to a deep neural network representer theorem (Section \ref{subsec:optimality_nn_representer_th}). 
In Section \ref{sec:analysis} we show the full flexibility of the new family of SDKNs by proving universal approximation properties and elucidating improved constructions which are made possible by using kernel mappings. \\
In particular, we will prove three universal approximation properties for the SDKN setup (see also Table \ref{tab:overview_approximation_results}):
\begin{enumerate}
\item Universal approximation in the number of centers: 
Is it possible to achieve a universal approximation for the SDKN by increasing the number of centers, but keeping the remaining architecture, i.e.\ width and 
depth fixed? See Section \ref{subsubsec:infinite_center_case}.
\item Universal approximation in the width of the network: 
Is it possible to achieve a universal approximation by increasing the width of the SDKN, i.e.\ increasing the dimension of the intermediate Euclidean spaces, 
but keeping the remaining architecture, i.e.\ depth and the number of centers fixed? See Section \ref{subsubsec:infinite_width_case}.
\item Universal approximation in the depth of the network: Is it possible to achieve a universal approximation by increasing the number of layers but keeping 
the remaining architecture, i.e.\ number of centers and widths fixed? See Section \ref{subsubsec:infinite_depth_case}.
\end{enumerate}

\begin{table}[h]
\begin{center}
\begin{tabular}{ |c|c|c| } 
 \hline
 & Exemplary kernels & Theorem reference \\ \hline
Unbounded center case & any universal kernels & Theorem \ref{th:unbounded_nr_centers} \\ 
Unbounded width case & e.g.\ Gaussian, basic Mat\'ern & Theorem \ref{th:unbounded_width} \\ 
Unbounded depth case & e.g.\ Gaussian, quadratic Mat\'ern & Theorem \ref{th:infinite_depth} \\ \hline
\end{tabular}
\caption{Overview of the approximation results in different asymptotic regimes.
}
\label{tab:overview_approximation_results}
\end{center}
\end{table}

For the sake of brevity we do not include experimental results in the present article, but we refer to \cite{wenzel2022structured,wenzel2024application} for examples of successful applications of SDKNs 
and comparisons to NNs in the context of closure term prediction and reduced order modeling.
We also have positive experience with the SDKN architecture in various other frameworks such as kernel autoencoders \cite{goehring2021bsc} or solving of ODEs \cite{wang2022bsc}, 
which are subject of ongoing work.
Code is available online.\footnote{\url{https://gitlab.mathematik.uni-stuttgart.de/pub/ians-anm/sdkn}}

\section{Kernel methods and neural networks}

We start by recollecting some background material that is required for the following discussion. \\

We consider an input domain $\Omega \subset \R^{d_0}$ and a target domain $Y \subset \R^{\dout}$ as well as a scattered data set $D = ((x_1, y_1), \ldots, 
(x_N, y_N))\in (\Omega\times Y)^N$. In the context of statistical learning theory, this data set is assumed to be distributed according to a typically
unknown probability distribution $\mathcal{P}$, i.e.\ $(x,y) \sim \mathcal{P}$.
The values $y_i$ can be distributed according to some deterministic target function $g$, i.e.\ $y_i = g(x_i)$, possibly affected by noise.
Formally, approximating this is usually done by minimizing a loss over a suitable space $\mathcal{H}$ of functions, for example by
\begin{align}
\label{eq:loss_functional}
f^* := \argmin_{f \in \calh} \mathcal{L}_D(f), \quad \mathcal{L}_D(f) := \frac{1}{N} \sum_{i=1}^N L (x_i, y_i, f(x_i)).
\end{align}
In the following we focus on the case of absence of noise.
The goal is to approximate a function $g$, which will be done in the following using the supremum norm $\|g\|_{\infty}:=\sup_{x\in\Omega}\left\|g(x)\right\|_{2}$. For this it is important that the space $\calh$ describes a sufficiently large set of functions. This will be discussed in Section \ref{subsec:universal_approx_prop}. Before, kernel methods and neural networks are introduced in Section \ref{subsec:approx_with_kernels} and Section \ref{subsec:neural_networks}. Section \ref{subsec:dk_repr_theorem} recalls the deep kernel representer theorem and finally Section \ref{subsubsec:related_work} provides an overview on related work.

\subsection{Approximation with kernel methods}
\label{subsec:approx_with_kernels}

Kernel methods are a popular and successful family of algorithms in approximation theory and machine learning, and they may be used to address the optimization problem of Eq.~\eqref{eq:loss_functional}.
We point the reader to \cite{Wendland2005, Fasshauer2015} for more background information. They are based on the use of a positive definite kernel, 
i.e., a symmetric function $k:\Omega\times\Omega\to\R$ such that the kernel matrix $(k(x_i, x_j))_{i, j=1}^N\in \R^{N\times N}$ is positive semi-definite for 
any set of $N$ points $X_N:=\{x_i\}_{i=1}^N\subset \Omega$, $|X_N| = N \in \N$. 
If the matrix is even positive definite whenever the points are pairwise distinct, the kernel is denoted strictly positive definite.
Given a %
positive definite kernel $k$ on a domain $\Omega$, there exists a unique reproducing kernel Hilbert space (RKHS).

Radial Basis Functions (RBFs) are a notable class of kernels that are defined as $k(x, z):=\varphi(\varepsilon\|x-z\|_2)$ with a width (or shape) parameter 
$\varepsilon>0$ and a univariate function $\varphi: [0, \infty)\to\R$. 
Some examples of RBFs that will be used in this work are given in Table \ref{tab:rbfs}.

\begin{table}
\begin{center}
\begin{tabular}{|l|c|}
\hline
& $\varphi(r)$\\
\hline
Gaussian & $\exp(-r^2)$\\
Basic Mat\'ern& $\exp(-r)$\\
Wendland of order $0$& $(1 -r)_+:=\max(1-r, 0)$\\
\hline
\end{tabular}
\caption{Examples of radial basis functions $\varphi$ which yield strictly positive definite kernels $k$ on $\Omega = \R$.}
\label{tab:rbfs}
\end{center}
\end{table}

A milestone result in kernel theory is the representer theorem (see \cite{Wahba1970,Schoelkopf2001v}). 
It ensures that if the loss optimization of Eq.~\eqref{eq:loss_functional} over the RKHS $\mathcal{H} = \ns$ has a solution, 
then there exists a minimizer $s$ that is actually an element of $V(X_N) := \Sp \{ k(\cdot, x), ~ x \in X_N \} \subset \ns$, i.e., 
\begin{align} \label{eq:approximant_kernel}
s(x)=\sum_{i=1}^N \alpha_i k(x, x_i).  
\end{align}
This representer theorem also allows for the use of a regularization term that depends monotonically on the RKHS norm of the function, i.e.\ $\mathcal{R}(f) = 
\Lambda(\Vert f \Vert_{\ns}^2)$ for some monotonically increasing function $\Lambda: \R_{\geq 0} \rightarrow \R_{\geq 0}$.
If $\Lambda$ is even strictly monotonically increasing, every minimizer admits such a representation. If also strong convexity of both the loss function $L$ 
and the regularization function $\Lambda$ is assumed, there is a unique solution of the form in Eq.\ \eqref{eq:approximant_kernel}. 
The infinite-dimensional optimization problem is thus reduced to a finite-dimensional one.

It can be shown that for each positive definite kernel $k$ there exists at least one feature map $\phi: \Omega \rightarrow \mathcal{H}$ into a feature space $\mathcal{H}$ such that $k(x, z)= \langle \phi(x), \phi(z) \rangle_{\mathcal{H}}$. Also the opposite holds, i.e.,  any function  $\tilde{k}(x, z) := \langle \phi(x), \phi(z) \rangle_{\mathcal{H}}$ is in fact a positive definite kernel on $\Omega$. Especially, it can be shown that kernel (ridge) regression is in fact linear regression applied to the data set $\left\{\left(\varphi(x_i), y_i\right)\right\}_{i=1}^N$ \cite{Goodfellow2016}, i.e., a simple algorithm (namely linear regression) is applied to the inputs that are transformed by means of the fixed feature map. This is partly a drawback of kernel methods, as they are limited by the fact that no data representation is learned from the data themselves, as instead is the case for neural networks, as we will discuss in the next section. Using our new method this limitation is removed.

Since the rise of deep learning, kernel methods have been considered to be less apt to process very large datasets, 
and indeed several limitations exists when using traditional kernel techniques.
Nevertheless, recent research has developed methods that allow a very efficient computational scalability of kernel methods. 
We refer for example to the Nystrom method \cite{Williams2001}, to Random Fourier Features \cite{Rahimi2008}, 
to EigenPro \cite{Ma2017, Ma2019b} and to Falkon \cite{Meanti2020}. 
The last paper in particular includes also a very recent survey of state-of-the-art methods that allow one to train kernel models with up to billions of data 
points.

\subsection{Approximation with neural networks}\label{subsec:neural_networks}

Neural network approaches are another popular and well-known machine learning technique.
In their simplest form, standard feedforward neural networks are given by a concatenation of simple functions as
\begin{align} \label{eq:approximant_nn}
f(x, \Theta) = W_L \sigma(\ldots(\sigma(W_2 \sigma(W_1x + b_1) + b_2))) + b_L
\end{align}
with weight matrices $W_i \in \R^{d_i \times d_{i-1}}, i=1, \ldots, L$ and bias terms $b_i \in \R^i, i=1, \ldots, L$ of suitable sizes and nonlinear activation 
functions $\sigma: \R^{d_i} \rightarrow \R^{d_i}$, that are applied componentwise and in order to simplify the notation, we do not indicate the dimensions in 
the notation $\sigma$. 
A common activation function is the $\relu$ function, which is given as $\sigma(x) = \max(x, 0)$. 
Due to this multilayer structure, neural networks are effective in approximating also high-dimensional data sets, 
simply by increasing the network size: 
Both the number of layers $L$ as well as the width determined by the inner dimensions $d_1, \ldots, d_{L-1}$ can be increased. 

Optimizing such a neural network means optimizing its parameters $\Theta := \{W_i, b_i, i=1, \ldots, L\}$ such that the loss is minimized. 
As the loss landscape in the parameters $\Theta$ is usually highly nontrivial, no direct optimization is feasible. 
Instead one usually relies on variants of stochastic gradient descent for an iterative minimization starting with randomly initialized weights. 

A lot of research has proved their practical success as well as their theoretical foundations: Universal approximation properties (see also Section 
\ref{subsec:universal_approx_prop}) have been proven \cite{math7100992}, 
the initialization of parameters was investigated \cite{yang2021tensor}, 
sophisticated optimization methods have been devised \cite{Goodfellow2016} and recent works have focussed on the overparametrized case, 
where the number of parameters $\Theta$ considerably exceeds the number of data points $N$ \cite{NEURIPS2019_62dad6e2}. 
Efficient frameworks like PyTorch \cite{NEURIPS2019_9015} enable the use and investigation of neural networks for a broad community.  
Nevertheless the theory of neural networks is not that well understood in comparison to kernel methods.

In view of Eq.\ \eqref{eq:approximant_nn}, the neural network can be viewed as a linear model after an optimizable feature map 
\begin{align} \label{eq:nn_feature_map}
x \mapsto \sigma(\ldots(\sigma(W_2 \sigma(W_1x + b_1) + b_2)).
\end{align}
This is a crucial advantage for high-dimensional data sets like images, where feature learning is crucial.

\subsection{Universal approximation properties} \label{subsec:universal_approx_prop}

A universal approximation property states that the class of functions generated from some model is dense in a considered function space, provided the setup is chosen large enough. 

For kernel methods, the question is whether functions described via Eq.\ \eqref{eq:approximant_kernel} for $N \rightarrow \infty$ are a sufficiently rich class of functions. This question is adressed via the notion of \textit{universal kernels}: A kernel is called universal, if its RKHS $\ns$ is dense in the Banach space of continuous function $C(\Omega)$. As the space of so called kernel translates, i.e.\ $\Sp\{k(\cdot, x), x \in \Omega \}$ is dense in the  RKHS $\ns$, this property affirmatively answers the question of universal approximation. We refer to \cite{Micchelli2006} for further details and remark, that common kernels like the radial basis function kernels given in Table \ref{tab:rbfs} are universal kernels.

For standard neural networks, the question is whether the parametrization of functions via Eq.\ \eqref{eq:approximant_nn} is sufficiently rich. This is 
confirmed for different limit cases: For the case of wide networks (that is $d_1, \ldots, d_{L-1} \rightarrow \infty$) see e.g.\ \cite{HORNIK1989359, 
LeshnoLPS93, pinkus1999approximation} and for the case of deep networks (that is $L \rightarrow \infty$) see e.g.\ \cite{pmlr-v49-eldan16, 
pmlr-v49-telgarsky16, math7100992, kidger2020universal, montanelli2020error}. 
We remark that a lot of recent approximation results for very deep neural networks have focussed on the popular \relu \ 
activation function.

But also the sigmoid activation function \cite{L21} or other smooth activation functions \cite{OK19} are 
frequently considered. Recent studies conclude universality
for ReLU networks of bounded width \cite{M21} and
for a wide class of activation functions \cite{BuiTan2024} using
the concept of neural network approximate identity.

\subsection{Deep kernel representer theorem} \label{subsec:dk_repr_theorem}

The following deep kernel representer theorem with slightly modified notation was stated and proved in \cite{Bohn2017}. It generalizes the standard representer 
theorem to multilayer kernel methods with matrix-valued kernels and thus constitutes a solid theoretical foundation.

\begin{theorem}[Deep kernel representer theorem]
\label{th:dk_representer_theorem}
Let $\mathcal{H}_1, \ldots, \mathcal{H}_L$ be reproducing kernel Hilbert spaces of functions with finite-dimensional domains $D_l$ and ranges $R_l \subseteq 
\R^{d_l}$ with $d_l \in \mathbb{N}$ for $l=1,\ldots, N$ such that $R_l \subseteq D_{l+1}$ for $l=1,\ldots, L-1, D_1 = \Omega$ 
and $R_L \subset \mathbb{R}$. 
Let furthermore $\mathcal{L}: \mathbb{R}^2 \rightarrow [0, \infty]$ be an abitrary loss function and let $\Theta_1, \ldots, \Theta_L: [0, \infty) \rightarrow 
[0, \infty)$ be strictly monotonically increasing functions. Then, a minimizing tuple $(f_l^*)_{l=1}^L$ with $f_l^* \in \mathcal{H}_l$ of
\begin{align}
J(f_1, \ldots, f_L) := \sum_{i=1}^N \mathcal{L}(y_i, f_L \circ \ldots \circ f_1(x_i)) + \sum_{l=1}^L \Theta_l(\Vert f_l \Vert_{\mathcal{H}_l}^2) \tag{$\star$}
\end{align}
fulfills $f_l^* \in \tilde{V} \subset \mathcal{H}_l$ for all $l=1, \ldots, L$ with 
\begin{align*}
\tilde{V}_l := \Sp \{ K_l(~ \cdot ~, f_{l-1}^* \circ \ldots \circ f_1^*(x_i))e_{k_l} | i=1,\ldots,N, k_l = 1,\ldots, d_l \},
\end{align*}
where $K_l$ denotes the reproducing kernel of $\mathcal{H}_l$ and $e_{k_l} \in \mathbb{R}^{d_l}$ is the $k_l$-th unit vector.
\end{theorem}

\begin{figure}[h]
\centering
\begin{tikzpicture}[node distance=.33cm and 1.0cm]
\node[] (start) {$\mathbb{R}^{d_{0}} \supset D_1$}; 
\node[right=of start] (layer1) {$\mathbb{R}^{d_1}$};
\draw[thick, ->] (start) to node[above, midway] {$f_1 \in \mathcal{H}_1$} (layer1);
\node[right=of layer1] (layer2) {$\mathbb{R}^{d_2}$};
\draw[thick, ->] (layer1) to node[above, midway] {$f_{2} \in \mathcal{H}_{2}$} (layer2);
\node[right=of layer2] (layer3) {...};
\draw[thick, ->] (layer2) to (layer3); 
\node[right=of layer3] (end) {$R_L \subset \R$}; %
\draw[thick, ->] (layer3) to node[above, midway] {$f_L \in \mathcal{H}_L$} (end);
\draw[thick, ->] (start) to[out=-15,in=-165] node[below, midway] {$F_2 := f_2 \circ f_1$} (layer2);
\end{tikzpicture}
\caption{Visualization of the mappings in Theorem \ref{th:dk_representer_theorem}.}
\label{fig:visualization_mappings}
\end{figure}

Note that the theorem can easily be generalized to vector-valued outputs. Using this theorem, the infinite-dimensional optimization problem with $f_l \in 
\mathcal{H}_l, l = 1, \ldots, L$ boils down to a finite-dimensional problem of optimizing the expansion coefficients $\alpha_i \in \R^{d_i}$ of 
\begin{align*}
\tilde{V}_l \ni f_l^* &= \sum_{i=1}^N K_l(\cdot, f_{l-1}^* \circ \ldots \circ f_1^*(x_i)) \alpha_i \\
\Rightarrow f_l^* \circ \ldots \circ f_1^*(\cdot) &= \sum_{i=1}^N K_l(f_{l-1}^* \circ \ldots \circ f_1^*(\cdot), f_{l-1}^* \circ \ldots \circ f_1^*(x_i)) 
\alpha_i.
\end{align*}
Like this the overall approximant is given by 
\begin{align*}
h(x) = f_L^*(\ldots(f_1^*(x))),
\end{align*}
which can be expressed as a standard kernel approximant $f^*(x) = \sum_{i=1}^N \mathcal{K}_L(\cdot,\tilde{x}_i) \alpha_i$ using a deep kernel $\mathcal{K}^L$
\begin{align}
\label{eq:deep_kernel}
\mathcal{K}^L(x, z) &= K_L(f_{L-1} \circ \ldots \circ f_1(x), f_{L-1} \circ \ldots \circ f_1(z)).
\end{align}
with the propagated centers $\tilde{x}_i = f_{L-1} \circ \ldots \circ f_1(x_i)$. Note that this kernel $\mathcal{K}^L$ is parametrized and learned by minimizing 
$(\star)$. %

\subsubsection{Notation}

In the following, we will drop the star-notation of the optimal functions and simply use $f_1, \ldots, f_L$. 
Furthermore we set $F_1 := f_1$, $F_j := f_j \circ \ldots \circ f_1$ for $j \geq 2$ such that $f(x)=f^*(x) = F_L(x)$. 
We use a superscript notation $x^{(i)}$ to refer to the $i$-th component in the standard Euclidean vector representation, 
and use the same notation also for the components of functions, i.e.\ $f_j^{(i)}(x)$ denotes the $i$-th component of $f_j(x)$. 
The setting of the representer theorem with the corresponding spaces is depicted in Figure \ref{fig:visualization_mappings}. 
As we will not focus on the inner domains $D_l$ and ranges $R_l$, they are not visualized explicitly, while the figure reports only the dimensions $d_0, d_1, 
\ldots, d_L$.

\subsection{Related work} \label{subsubsec:related_work}

In order to allow for the flexibility of learning feature maps instead of using a fixed one, 
several multilayer kernel methods have been proposed. 
This is done frequently based on neural networks: In \cite{Le2016}, a neural network is used in order to directly learn a kernel function from the data. 
The paper \cite{Wilson2016} suggests to use a neural network to transform the data,
and subsequently use a fixed kernel on top. 
Like this the deep neural network is used to learn a feature representation, and both parts are optimized jointly. 
Another approach that combines kernel methods and neural networks is proposed in \cite{ijcai2019-558}, which introduces a setup called \textit{Deep Spectral Kernel Learning}. 
Here a multilayered kernel mapping is built by using Random Fourier Feature mappings in every layer. 
This is essentially a neural network with fixed parameters in every second layer. 
Recently, \cite{wurzberger2024learning} combined radial basis function kernels with ideas from convolutional neural networks to obtain deep radial basis function networks.

Hierarchical Gaussian kernels, which are an iterated composition of Gaussian kernels, are introduced in \cite{steinwart2016learning}. The setup has as well a 
structure comparable with neural networks. However, the methodology does not easily extend to a broader class of kernels.
From a Gaussian process point of view, \cite{damianou2013deep} introduces Deep Gaussian processes, which is a hierarchical composition of Gaussian processes. 
In order to scale them to large data sets, approximation techniques are used. 
Under the notion of \textit{Kernel Flows}, \cite{OWHADI201922,hamzi2021learning, hamzi2021simple, hamzi2023learning, lee2023learning
} introduces a way to build a kernel step-by-step, by iteratively modifying the transformation of the data, which is described with help of standard kernel 
mappings.

\textit{Arc-cosine} kernels are introduced in \cite{NIPS2009_3628}, which the authors also call multilayer kernel machines, that emerge in the unbounded width limit of neural networks. Further research in this direction is conducted under the notion of the \textit{neural tangent kernel} \cite{jacot2018neural}. \\
Starting from the neural network point of view, there are approaches that do not use fixed activation functions but strive for optimized ones. 
The paper \cite{JMLR:v20:18-418} extends the loss functional with a specific regularization term. This allows the authors to derive a representer theorem, yielding activation functions given by nonuniform piecewise linear splines. Another approach coined \textit{Kafnets} is introduced in  \cite{SCARDAPANE201919}. 
Their structure of the network is quite similar to the setup which we will derive based on the deep kernel representer theorem. 
However in contrast to ours, 
their setup is built on prechosen center points within each layer, given e.g.\ by a grid that is independent of the data. 
The article \cite{liu2024kan} recently introduced Kolmogorov-Arnold networks as an alternative to neural networks,
leveraging splines to obtain optimizable activation functions. \\

We remark that \cite{ijcai2019-558, JMLR:v20:18-418} draw some links to the deep kernel representer theorem, but do not make use of it. 
By using the deep kernel representer theorem for our SDKN, we obtain optimality of our proposed setup in contrast to other approaches like standard neural networks, that enjoy the same structure. In Section \ref{subsec:optimality_nn_representer_th} we further comment on this and especially elaborate on the connections and differences to the neural network representer theorem of \cite{JMLR:v20:18-418}.

\section{Structured Deep Kernel Networks} \label{sec:sdkn}

In the first Section \ref{subsec:rbf_approach}, we comment on a straightforward application of the deep kernel representer theorem (Theorem \ref{th:dk_representer_theorem}) using common classes of kernels and elaborate why this approach is not suitable. 
Subsequently, we introduce our Structured Deep Kernel Network (SDKN) approach in Section \ref{subsec:setup_sdkn} and comment on its optimality and relation to neural networks in Section \ref{subsec:optimality_nn_representer_th}.

\subsection{Radial basis function approach} \label{subsec:rbf_approach}

In order to use kernels in a multilayer setup according to \Cref{th:dk_representer_theorem},
one could directly leverage matrix-valued radial basis function kernels like the Gaussian, Wendland or Mat\'ern kernels, i.e.\
\begin{align*}
k(x,z) = k_{\text{Gauss}}(x,z) \cdot I_d \in \R^{d \times d}
\end{align*}
or even more general matrix-valued kernels.
In preliminary numerical experiments we have seen that such an approach is not successful, 
as there are way too many parameters to optimize that do not obey any special structure. 
To reduce the number of optimizable parameters and simplify the optimization, one can pick a subset $\{ z_i \}_{i=1}^M \subset \{ x_i \}_{i=1}^N$ for $M \ll N$ and use them as center points, such that every mapping $f_l: \R^{d_l} \rightarrow R^{d_{l+1}}$ is given by
\begin{align*}
f_l(\cdot) = \sum_{i=1}^M  k_l(\cdot , F_{l-1}(z_i)) \alpha_i
\end{align*}
with optimizable parameters $\alpha_i \in \R^{d_{l+1}}$. But even this approach is experimentally inferior to the setup of SDKNs introduced below. %
The theoretical reason for this is likely the choice of ansatz (trial) space.
Indeed, assume that the output of the function $f_l$ should be slightly adjusted, s.t. for $F_{l-1}(z_i)$ it holds
\begin{align*}
f_l(F_{l-1}(z_i)) &= y_i + \epsilon \\
\text{instead of } \hspace{1cm} f_l(F_{l-1}(z_i)) &= y_i,
\end{align*}
with $\epsilon \in \R^{d_{l+1}}$. Due to the potentially very bad conditioning of the kernel matrix $A \in \R^{Md_l \times Md_l}, A_{ij} = k_l(F_{l-1}(z_i), 
F_{l-1}(z_j))$ (see e.g.\ \cite{diederichs2019improved}), this results in large changes for the parameters $\alpha_j$ even for small $\epsilon$. This slows down 
the optimization when using common stochastic gradient descent optimization strategies \cite{Ma2017}. \\
A remedy for this might be a change of basis, e.g.\ using a Lagrange or Newton basis instead of the basis of kernel translates $\{k(\cdot, z_i), i=1, .., M\}$ \cite{MULLER2009645}. However, this is not possible in deep setups with more than two layers, as the Lagrange and Newton basis both depend on the used centers $\{ z_i \}_{i=1}^M$, and these change as soon as the previous functions $f_1, .., f_{l-1}$ change. \\

\subsection{Setup of structured deep kernel network} \label{subsec:setup_sdkn}

In order to alleviate these problems, we propose a better setup with a more subtle selection of kernels, which
will be called Structured Deep Kernel Network (SDKN).
For this, in the following we introduce two types of kernels, namely simple linear kernels (Section \ref{subsubsec:linear_kernel}) and kernels that are acting 
on single dimensions (Section \ref{subsubsec:single_dim_kernel}). 
Both types of kernels can be proven to be positive definite, but in general not strictly positive definite (see \Cref{app:positive_definiteness}). 
Subsequently in Section \ref{subsubsec:overall_setup}, we combine these types of kernels in an alternating way.

\subsubsection{Linear kernel} \label{subsubsec:linear_kernel}

The linear kernel $k_\text{lin}: \R^d \times \R^d \rightarrow \R$ is defined via $(x,z) \rightarrow \langle x, z \rangle_{\R^d}$, i.e.\ the standard Euclidean 
inner product is computed. We have the following theorem:

\begin{prop} \label{prop:linear_layer}
Given $A \in \R^{b \times d}$,
the linear mapping $\R^{d} \rightarrow \R^b, x \mapsto Ax$ can be realized as a kernel mapping 
\begin{align*}
s: \R^d \rightarrow \R^b, 
x \mapsto \sum_{i=1}^M k(x, z_i) \alpha_i, ~~ \alpha_i \in \R^b
\end{align*} 
with centers $\{ z_i \}_{i=1}^M \subset \R^d$ and using the matrix-valued linear kernel 
\begin{align}
k(x,z) = k_\text{lin}(x,z) \cdot I_b = \langle x, z \rangle_{\R^d} \cdot I_b, \tag{$\star_1$}
\end{align}
if and only if the span of the center points $\{z_i\}_{i=1}^M \subset \R^d$ is a superset of the row space of the matrix A.
\end{prop}
A proof is given in the appendix, see \proofref{prop:linear_layer}. We remark that \Cref{prop:linear_layer} can also be generalized to the mapping $x \mapsto Ax + c$, $c \in \R^b$, i.e.\ to a linear mapping with bias. However, as we will focus on translational invariant kernels in the following, such an additional bias term has no effect.  \\
We proceed with the following proposition, which essentially shows that any linear mapping is possible, i.e.\ it is not required to check that the span of the 
center points is a specific superset as required in \Cref{prop:linear_layer}: %

\begin{prop} \label{prop:linear_layer_2}
Let $d_0,d\in\N$, $\Omega\subset\R^{d_0}$ and consider a mapping $g: \Omega \to\R^d$ with $x\mapsto g(x):=(g^{(1)}(x), \ldots, g^{(d)}(x))$ and $g^{(i)}: \Omega 
\rightarrow \R$, $1\leq i\leq d$. Denote as $g(\Omega)$ its image.

Then any linear combination
\begin{align*}
(g^{(1)}(x), \ldots, g^{(d)}(x))^\top \mapsto \sum_{j=1}^d g^{(j)}(x) \cdot \beta_j
\end{align*}
with $\beta_j \in \R^b, j=1, \ldots, d$ can be realized with help of a linear kernel $k$ using propagated centers $g(z_i) \in g(\Omega)$, i.e.\ 
\begin{align*}
&\exists z_1, \ldots, z_d \in \Omega, \alpha_1, \ldots, \alpha_d \in \R^b ~ \forall x \in \Omega \\ 
&\qquad \qquad \sum_{j=1}^d g^{(j)}(x) \cdot \beta_j = \sum_{i=1}^d k(g(x), g(z_i)) \cdot \alpha_i
\end{align*}
\end{prop}

The proof is given in the appendix, see \proofref{prop:linear_layer_2}. 
Note that \Cref{prop:linear_layer_2} means that any linear mapping can be obtained using the linear kernel.
The following example illustrates the statement of the proposition:

\begin{example}
Choose $d_0 = 1, d=3$ and $\Omega = [0,1]$ as well as $g: \R \rightarrow \R^3$ with $g^{(1)}(x) = x, g^{(2)}(x) = 2x, g^{(3)}(x) = x^2$. 
Consider a set $X_N \subset \Omega$ with at least 
three pairwise distinct points, i.e.\ $N \geq 3$. Then the space $\Sp \{ g(z_i), z_i \in X_N \} \subseteq \R^d$ spanned by the propagated data points $g(z_i)$ 
is only two-dimensional as $(g^{(1)}(z_i))_{i=1}^N$  and $(g^{(2)}(z_i))_{i=1}^N$ are linearly dependent. In particular, it holds $\Sp \{ g(z_i) , z_i \in X_N  
\} = \Sp \{ (1, 2, 0)^\top, (0, 0, 1)^\top \}$. \\
Thus, according to \Cref{prop:linear_layer} it is not possible to realize the linear combination $(z^{(1)}, z^{(2)}, z^{(3)}) \mapsto 2 \cdot z^{(1)} - z^{(2)}$.
However, note that $2 \cdot g^{(1)}(x) - g^{(2)}(x) = 2x - 2x = 0$.
\end{example}

\subsubsection{Single-dimensional kernels} \label{subsubsec:single_dim_kernel}

The following type of kernel $k_s: \R^d \times \R^d \rightarrow \R^{d \times d}$, that is composed of $d$  strictly positive definite kernels $k^{(1)}, \ldots, 
k^{(d)}$ acting on one-dimensional inputs, will be called \textit{(matrix valued) single-dimensional kernel}, or componentwise kernels (the index $s$ refers to 
\textit{single}):
\begin{align}
\label{eq:special_kernel}
k_s: \R^d \times \R^d \rightarrow \R^{d \times d}, k_s(x,z) = 
\begin{bmatrix}
   k^{(1)}(x^{(1)}, z^{(1)}) &  & 0 \\ 
   &  \ddots & \\ 
   0 &   & k^{(d)}(x^{(d)}, z^{(d)})
 \end{bmatrix}
\end{align}
Recall that the notation $x^{(i)}, z^{(i)}$ means that only the $i$-th component of the vectors $x$ or $z$ is used. %
In the following we will mostly choose the same kernel for each component, i.e.\ $k^{(1)} = \ldots = k^{(d)}$, but this general setting is possible as well.

\subsubsection{Overall setup} \label{subsubsec:overall_setup}

Following the notation within the deep kernel representer theorem (Theorem \ref{th:dk_representer_theorem}), 
we seek an approximant of the form $f_L \circ \ldots \circ f_1(\cdot)$. 
We will use $L$ odd and use alternately the linear kernel from Section \ref{subsubsec:linear_kernel} and the single-dimensional kernels from Section \ref{subsubsec:single_dim_kernel}, 
where we always start (and therefore end) with the linear kernel. 
This results in an alternating concatenation of linear and nonlinear functions.
The dimensionality of the output of these matrix-valued kernels directly corresponds to the dimensionality of the hidden spaces and can be chosen arbitrarily. 

The overall setup is therefore depicted in Figure \ref{fig:visualization_NN_DK} and is reminiscent of the architecture of standard feedforward neural networks:
\begin{enumerate}
\item The mappings $f_l$ for $l$ odd use the linear kernel ($\star_1$) and thus can realize linear mappings $x \mapsto Ax$ via \Cref{prop:linear_layer}, %
especially any linear mapping that is compatible with the data in the sense of \Cref{prop:linear_layer_2} can be realized.
\item The mappings $f_l$ for $l$ even use the single-dimensional kernels ($\star_2$). Thus these mappings have the form
\begin{align*}
f_l^{(j)}(x) = \sum_{i=1}^N \alpha_{l,i}^{(j)} k_j(x^{(j)}, F_{l-1}(z_i)^{(j)}),
\end{align*}
with $\alpha_{l,i}^{(j)} \in \R$.
These functions $f_l^{(j)}$ for $l=2, 4, \ldots$, $j=1, \ldots, d_l$ can be viewed as nonlinear activations acting componentwise, which are also present in 
neural networks. 
However, as the parameters $\alpha_{l,i}^{(j)} \in \R$ are subject to optimization, these functions are \textit{optimizable activation functions}. This is a conceptual advantage over NNs that mostly use fixed or parametrized activation functions. \\
In order to enhance sparsity of the model, 
we only use $M \ll N$ training points (called \textit{centers}) of the whole data set $X_N$, i.e.\ the remaining coefficients are set to zero at initialization and remain fixed. Thus we have
\begin{align}
\label{eq:optim_activation_function}
f_l^{(j)}(x) = \sum_{i=1}^M \alpha_{l,i}^{(j)} k_j(x^{(j)}, F_{l-1}(z_i)^{(j)}),
\end{align}
using fixed, (randomly) prechosen centers $z_i \in X_N$.
 This sparsity speeds up the optimization procedure in a practical implementation and does not impede the approximation capabilites as analyzed in Section \ref{sec:analysis}.
\end{enumerate}

\subsection{Generality, sparsity and optimality} \label{subsec:optimality_nn_representer_th}

We want to highlight some positive properties of our SDKN approach,
particularly as it generalizes some existing results.

\begin{itemize}
\item[(a)] Generality of kernels:
  We do not pose severe restrictions to the applicable kernels in the different SDKN layers, hence allow a wide range of kernels of
  different smoothness, locality, etc.   
\item[(b)] Sparsity of the SDKN:
  We typically choose a small number $M\ll N$ of centers, which
  is sufficient in theory (see \Cref{th:unbounded_width} and \Cref{th:infinite_depth}) and practice \cite{wenzel2024application,wenzel2022structured}.
  By this we immediately obtain sparse expansions of the SDKN model.
  In practice we realize this by simply choosing the corresponding
  coefficients for optimization and setting the remaining ones to zero
  (as explained in Subsection \ref{subsubsec:overall_setup}). 
  The particular choice of centers seems of less relevance in practice.
\item[(c)] Optimality of propagated centers: The deep kernel
  representer theorem directly states the optimality of the propagated
  centers. This drastically reduces the number of parameters to be chosen
  in the architecture. 
\end{itemize}

\begin{rem}[Relation to NN representer theorem]
  We remark that the paper \cite{JMLR:v20:18-418} derives a
representer theorem for deep neural networks based on the use of a particular 
second-order total variation regularization term. Its use is motivated by the
idea of promoting activation functions with favorable propertis, namely they 
should be \textit{simple} for practical relevance and
\textit{piecewise linear}, because this works in practice. 
The resulting optimal activations, for each node, are shown to be piecewise
linear splines.
This representer theorem is comparable to the deep kernel representer
theorem \Cref{th:dk_representer_theorem} when 
combined with the setup of the SDKN,
and in this respect our approach has
some additional generality. %
In view of item (a) we can cover the piecewise linear spline
kernels in our SDKN architecture, hence generalize that architecture. 
But the SDKN structure
also allows the use of other kernels,
while preserving its inherent optimality.
This furthermore also yields the optimality of the linear layers,
which were not discussed in \cite{JMLR:v20:18-418}.
With respect to sparsity, i.e.\ item (b), 
\cite{JMLR:v20:18-418} leverages a function norm for regularization, 
that turns out to be an $\ell^1$-norm of coefficients, 
and thus can be used to promote sparsity,
i.e.\ reducing the number of nonzero parameters from $N$ to $K \ll N$.
However there is no explicit control of the number of nonzero terms
in contrast to our explicit first layer center choice.
The item (c) is a particular advantage
compared to \cite{SCARDAPANE201919}, where fixed (non-optimized) centers are
used within each layer.
\end{rem}

\section{Analysis of approximation properties} \label{sec:analysis}

Given a particular setup of the SDKN,
the mapping which is realized by that SDKN is inherently based on the provided (training) data $X_N$ due to its construction based on the representer  \Cref{th:dk_representer_theorem}, 
from which its optimality is given. 
Thus, without any data, i.e.\ $X_N = \emptyset$, only the zero mapping can be realized.  
Provided there is sufficient training data $X_N$, 
it will be shown in the following that the SDKN satisfies universal approximation results in different asymptotic regimes: 
We assume $X_N = \Omega$, but require only $M=3$ centers $z_1, z_2, z_3 \in \Omega$ for the optimization of the single-dimensional kernel function layers. 
In what follows, we stress that only the structure of the SDKN is based on the representer  \Cref{th:dk_representer_theorem}, while the coefficients in the 
kernel expansion are set to convenient values, and are not necessarily the result of a 
data-dependent optimization process.

We remark that proving the universal approximation properties is not possible by reusing standard NN approximation statements, for two reasons: 
First, the mapping of NNs are decoupled from the (training) data 
and second they use fixed pre-chosen activation functions instead of optimizable activation functions (see Section \ref{subsubsec:overall_setup}). 
Thus we need new tools and arguments.

In order to derive such statements, the focus will be on the approximation of scalar-valued outputs. 
The extension to vector valued outputs can be accomplished in a straightforward way by treating the output dimensions separately. We further restrict to $\Omega = [0, 1]^d \subset \R^d$, but we remark that generalizations to more general compact sets $\Omega$ are possible, using transformation arguments. 

\noindent We will make use of the following notation:
\begin{enumerate}
\item $\mathcal{F}_{L,w,M}$ denotes the class of functions from $\Omega \subset \R^{d_0}$ to $\R^{d_{2L+1}}$ that can be realized by any SDKN with $L \in \N$ 
optimizable nonlinear activation function layers and $L+1$ linear layers.
Further, $w:= \max(d_0, \ldots, d_{2L+1}) \in \N$ denotes the maximal dimension and $M$ refers to the number of centers. 
Due to the comparability with neural networks, $w$ will also be called the \textit{width} of the SDKN, while $L$ will be called its \textit{depth}. Note that the number of layers therefore is $2L+1$. 
An example for this notation is also given in Figure \ref{fig:visualization_NN_DK}.
In the following we restrict to the scalar-valued case, i.e.\ $d_{2L+1} = 1$.
\item $\dist(\varphi, \mathcal{F}) := \inf_{f \in \mathcal{F}} \Vert \varphi - f \Vert_{L^\infty(\Omega)}$ for $\varphi \in C(\Omega)$ for some $\mathcal{F} \subset C(\Omega)$. 
Note that $\mathcal{F}$ is not necessarily closed, thus we consider $\inf_{f \in \mathcal{F}}$ instead of $\min_{f \in \mathcal{F}}$.
\item Starting with Figure \ref{fig:visualization_NN_DK_inf_centers},
we visualize small SDKN architectures that can be used to approximate simple functions.
\end{enumerate}

\subsection{Universal approximation in the number of centers} \label{subsubsec:infinite_center_case}

In order to show that the proposed structured deep kernel setup satisfies a universal approximation property in the number of centers, 
we recall a version of the Kolmogorov-Arnolds theorem \cite{lorentz1996constructive}. %

\begin{theorem}[Kolmogorov-Arnolds Theorem]
\label{th:kolmogorov_theorem}
There exist $d$ constants $\lambda_j > 0, j=1, \ldots, d$ with $\sum_{j=1}^d \lambda_j \leq 1$, 
and $2d+1$ continuous, strictly increasing functions $\phi_q$, $q=0, \ldots., 2d$, which map $[0,1]$ to itself, such that every continuous function $f$ of $d$ 
variables on $[0,1]^d$ can be represented in the form 
\begin{align*}
f(x) = f(x^{(1)}, \ldots, x^{(d)}) = \sum_{q=0}^{2d} \Phi \left( \sum_{j=1}^d \lambda_j \phi_{q}(x^{(j)}) \right)
\end{align*}
for some $\Phi \in C([0,1])$, depending on $f$.
\end{theorem}

Using this theorem, it is possible to show that any continuous function $f: [0,1]^d \rightarrow \R$ can be approximated to arbitrary accuracy by an SDKN setup of fixed inner dimension and depth. 
An exemplary setup for input dimension $d=2$ is visualized in Figure \ref{fig:visualization_NN_DK_inf_centers}. 
The details are given in Theorem \ref{th:unbounded_nr_centers}, and its assumptions are collected in the following. 
In particular, a setup using $L=2$ is sufficient to derive a universal approximation statement:

\begin{assumption}%
[Universal approximation in the number of centers]
\label{ass:inf_centers_assumptions} ~
\begin{enumerate}
\item The kernels $k_2, k_4$ used in the single-dimensional mappings are universal kernels, i.e.\ their RKHS are dense in the space of continuous functions.
\item The number of centers $M\in \N$ and the centers can be chosen arbitrarily within $[0,1]^d$ (unbounded number of centers).
\end{enumerate}
\end{assumption}

\begin{figure}[H]
\def\layersep{2cm}
\centering %
\begin{tikzpicture}[scale=0.5, shorten >=1pt,->,draw=black!50, node distance=\layersep]
    \tikzstyle{every pin edge}=[<-,shorten <=1pt]
    \tikzstyle{neuron}=[circle,fill=black!25,minimum size=8pt,inner sep=0pt]
    \tikzstyle{input neuron}=[neuron, fill=green!50];
    \tikzstyle{output neuron}=[neuron, fill=red!50];
    \tikzstyle{hidden neuron}=[neuron, fill=blue!50];
    \tikzstyle{annot} = [text width=4em, text centered]

	\node[annot] (hl1) at (-3, -5.5) {Input};
	\node[annot] (hl2) at (19, -5.5) {Output};

    \foreach \name / \y in {1,2}
        \node[input neuron,pin={[pin edge={-}]left:}] (I-\name) at (0,-4-\y) {};

	\node[hidden neuron, label={[label distance=-.1cm]90:\tiny{$x^{(1)}$}}] (H1-1) at (1.5*\layersep,-1) {};
	\node[hidden neuron, label={[label distance=-.1cm]90:\tiny{$x^{(2)}$}}] (H1-2) at (1.5*\layersep,-2) {};
	\node[hidden neuron, label={[label distance=-.1cm]90:\tiny{$x^{(1)}$}}] (H1-3) at (1.5*\layersep,-3) {};
	\node[hidden neuron, label={[label distance=-.1cm]90:\tiny{$x^{(2)}$}}] (H1-4) at (1.5*\layersep,-4) {};
	\node[hidden neuron, label={[label distance=-.1cm]90:\tiny{$x^{(1)}$}}] (H1-5) at (1.5*\layersep,-5) {};
	\node[hidden neuron, label={[label distance=-.1cm]90:\tiny{$x^{(2)}$}}] (H1-6) at (1.5*\layersep,-6) {};
	\node[hidden neuron, label={[label distance=-.1cm]90:\tiny{$x^{(1)}$}}] (H1-7) at (1.5*\layersep,-7) {};
	\node[hidden neuron, label={[label distance=-.1cm]90:\tiny{$x^{(2)}$}}] (H1-8) at (1.5*\layersep,-8) {};
	\node[hidden neuron, label={[label distance=-.1cm]90:\tiny{$x^{(1)}$}}] (H1-9) at (1.5*\layersep,-9) {};
	\node[hidden neuron, label={[label distance=-.1cm]90:\tiny{$x^{(2)}$}}] (H1-10) at (1.5*\layersep,-10) {};        

    \node[hidden neuron, label={[label distance=-.1cm]90:\tiny{$\tilde\phi_0(x^{(1)})$}}] (H2-1) at (3*\layersep,-1) {};
    \node[hidden neuron, label={[label distance=-.1cm]90:\tiny{$\tilde\phi_0(x^{(2)})$}}] (H2-2) at (3*\layersep,-2) {};
    \node[hidden neuron, label={[label distance=-.1cm]90:\tiny{$\tilde\phi_1(x^{(1)})$}}] (H2-3) at (3*\layersep,-3) {};
    \node[hidden neuron, label={[label distance=-.1cm]90:\tiny{$\tilde\phi_1(x^{(2)})$}}] (H2-4) at (3*\layersep,-4) {};
    \node[hidden neuron, label={[label distance=-.1cm]90:\tiny{$\tilde\phi_2(x^{(1)})$}}] (H2-5) at (3*\layersep,-5) {};
    \node[hidden neuron, label={[label distance=-.1cm]90:\tiny{$\tilde\phi_2(x^{(2)})$}}] (H2-6) at (3*\layersep,-6) {};
    \node[hidden neuron, label={[label distance=-.1cm]90:\tiny{$\tilde\phi_3(x^{(1)})$}}] (H2-7) at (3*\layersep,-7) {};
    \node[hidden neuron, label={[label distance=-.1cm]90:\tiny{$\tilde\phi_3(x^{(2)})$}}] (H2-8) at (3*\layersep,-8) {};
    \node[hidden neuron, label={[label distance=-.1cm]90:\tiny{$\tilde\phi_4(x^{(1)})$}}] (H2-9) at (3*\layersep,-9) {};
    \node[hidden neuron, label={[label distance=-.1cm]90:\tiny{$\tilde\phi_4(x^{(2)})$}}] (H2-10) at (3*\layersep,-10) {};            

	\node[hidden neuron, label={\tiny{$\sum \lambda_j \tilde\phi_0(x^{(j)})$}}] (H3-1) at (4.5*\layersep,-2.5) {};
	\node[hidden neuron, label={\tiny{$\sum \lambda_j \tilde\phi_1(x^{(j)})$}}] (H3-2) at (4.5*\layersep,-4) {};
	\node[hidden neuron, label={\tiny{$\sum \lambda_j \tilde\phi_2(x^{(j)})$}}] (H3-3) at (4.5*\layersep,-5.5) {};
	\node[hidden neuron, label={\tiny{$\sum \lambda_j \tilde\phi_3(x^{(j)})$}}] (H3-4) at (4.5*\layersep,-7) {};
	\node[hidden neuron, label={\tiny{$\sum \lambda_j \tilde\phi_4(x^{(j)})$}}] (H3-5) at (4.5*\layersep,-8.5) {};	        

	\node[hidden neuron, label={\tiny{$\tilde\Phi(\ldots)$}}] (H4-1) at (6*\layersep,-2.5) {};
	\node[hidden neuron, label={\tiny{$\tilde\Phi(\ldots)$}}] (H4-2) at (6*\layersep,-4) {};
	\node[hidden neuron, label={\tiny{$\tilde\Phi(\ldots)$}}] (H4-3) at (6*\layersep,-5.5) {};
	\node[hidden neuron, label={\tiny{$\tilde\Phi(\ldots)$}}] (H4-4) at (6*\layersep,-7) {};
	\node[hidden neuron, label={\tiny{$\tilde\Phi(\ldots)$}}] (H4-5) at (6*\layersep,-8.5) {};	

    \foreach \name / \y in {1}
    \node[output neuron,pin={[pin edge={-}]right:}, label={\tiny{$\sum_{q=0}^{2d} \tilde \Phi(\ldots)$}}] (O-\name) at (8*\layersep,-5.5) {};    

	\foreach \dest in {1,3,5,7,9}
		\path [lightgray] (I-1) edge (H1-\dest);    
	\foreach \dest in {2,4,6,8,10}
		\path [lightgray] (I-2) edge (H1-\dest);

    \path [lightgray] (H2-1) edge (H3-1);
	\path [lightgray] (H2-3) edge (H3-2);
	\path [lightgray] (H2-5) edge (H3-3);
	\path [lightgray] (H2-7) edge (H3-4);
	\path [lightgray] (H2-9) edge (H3-5);

    \path [lightgray] (H2-2) edge (H3-1);
	\path [lightgray] (H2-4) edge (H3-2);
	\path [lightgray] (H2-6) edge (H3-3);
	\path [lightgray] (H2-8) edge (H3-4);
	\path [lightgray] (H2-10) edge (H3-5);

    \foreach \source in {1,...,5}
	    \foreach \dest in {1,...,1}
	        \path [lightgray] (H4-\source) edge (O-\dest);    

	\foreach \node in {1,...,10}
		\path [black] (H1-\node) edge (H2-\node);
		
	\foreach \node in {1,...,5}
		\path [black] (H3-\node) edge (H4-\node);

\end{tikzpicture}
\caption{Visualization for the unbounded number of centers case for $d = d_0 = 2, d_{2L+1} = 1$: Unbounded number of centers, but otherwise fixed setup. The network is mimicking the Kolmogorov-Arnolds decomposition of a target function according to
  Thm.\ \ref{th:kolmogorov_theorem} and approximating (indicated by $\tilde{}$)
  the corresponding functions $\phi_i$ and $\Phi$.
The first, third and fifth mapping use a linear kernel, 
the second and fourth one use single dimensional kernels $k_2, k_4$. 
Within the linear kernel layers, not all connections are required. 
The first layer just duplicates the inputs, the second layer approximates the mappings $\phi_q(\cdot)$, the third layer
builds the linear combination of the approximants $\sum_{j=1}^d \lambda_j \tilde\phi_{q}(x^{(j)})$ and the
fourth layer approximates the mappings $\Phi$ while the last layer builds the sum $\sum_{q=0}^{2d} \tilde\Phi(\ldots)$.}
\label{fig:visualization_NN_DK_inf_centers}
\end{figure}

\begin{theorem}[Universal approx.\ for unbounded number of centers] \label{th:unbounded_nr_centers}
For any $d \in \N$, 
let $\Omega = [0,1]^{d}$ and consider an arbitrary continuous function $f: \Omega \rightarrow \R$. 
Then it is possible under \Cref{ass:inf_centers_assumptions} to approximate this function $f$ to arbitrary accuracy using an SDKN
of finite width $w = (2d+1)d$ and finite depth $L=2$, i.e.\
\begin{align*}
  \lim_{M \rightarrow \infty} \dist(f, \mathcal{F}_{2, (2d+1)d, M}(\Omega)) = 0.
\end{align*}
\end{theorem}

As the proof is rather technical it can be found in the appendix, see \proofref{th:unbounded_nr_centers}
This approach via the unbounded number of centers case rather is of theoretical nature than
of practical use, see also e.g.\ \cite{girosi1989representation} for the same case in neural networks. 
In the SDKN case, the challenge is to approximate the mappings $\phi_{q}$ and $\Phi$ to sufficient precision. 
However, in practice this is difficult due to stability issues for kernel methods in the case of large number of centers, see also the explanations
in Section \ref{subsec:rbf_approach}. %

\subsection{Universal approximation in the width of the network} 
\label{subsubsec:infinite_width_case}

In the case of unbounded width, we consider an SDKN setup using $L=1$, 
i.e. $f_3 \circ f_2 \circ f_1: \R^{d_0} \rightarrow \R^{d_3}$ with $f_l: \R^{d_{l-1}} \rightarrow \R^{d_l}, l=1,2,3$. 
The mappings $f_1$ and $f_3$ describe linear mappings, $f_2$ is a single-dimensional kernel mapping. 
While $d_0$ and $d_3$ are fixed and given by the input and output dimension of the learning problem, 
$d_1 = d_2$ are unbounded and we consider the case $d_1 = d_2 \rightarrow \infty$.
The situation is visualized in Figure \ref{fig:visualization_NN_DK_inf_width}. Thus the overall model considered here is given by 
\begin{align*}
&f(x) = A_3f_2(A_1x), \qquad A_1 \in \R^{d_1 \times d_0}, A_3 \in \R^{d_3 \times d_2}, \\
&f_2^{(j)}(x) = \sum_{i=1}^M \alpha_{2,i}^{(j)} k^{(j)}(x^{(j)}, (A_1 z_i)^{(j)}), ~ 1 \leq j \leq d_1 = d_2.
\end{align*}

\begin{figure}[H]
\def\layersep{2cm}
\centering %
\begin{tikzpicture}[scale=0.5, shorten >=1pt,->,draw=black!50, node distance=\layersep]
    \tikzstyle{every pin edge}=[<-,shorten <=1pt]
    \tikzstyle{neuron}=[circle,fill=black!25,minimum size=8pt,inner sep=0pt]
    \tikzstyle{input neuron}=[neuron, fill=green!50];
    \tikzstyle{output neuron}=[neuron, fill=red!50];
    \tikzstyle{hidden neuron}=[neuron, fill=blue!50];
    \tikzstyle{annot} = [text width=4em, text centered]

	\node[annot] () at (-2.5, -2.5) {Input};
	\node[annot] () at (9, -2.5) {Output};

    \foreach \name / \y in {1,2,3}
        \node[input neuron,pin={[pin edge={-}]left:}] (I-\name) at (0,-.5-\y) {};

    \foreach \name / \y in {1,2,3,9,10}
        \node[hidden neuron] (H1-\name) at (\layersep,3-\y) {};
        
    \foreach \name / \y in {1,2,3,9,10}
        \node[hidden neuron] (H2-\name) at (2*\layersep,3-\y) {};

    \foreach \name / \y in {1}
        \node[output neuron,pin={[pin edge={-}]right:}] (O-\name) at (3*\layersep,-1.5-\y) {};    

    \node(dots) at (1.5*\layersep,-2.5){\vdots};

    \foreach \source in {1,2,3}
	    \foreach \dest in {1,2,3,9,10}
	        \path [lightgray] (I-\source) edge (H1-\dest);    
    
	\foreach \source in {1,2,3,9,10}
		\foreach \dest in {1}
	        \path [lightgray] (H2-\source) edge (O-\dest);   

	\foreach \node in {1,2,3,9,10}
		\path [black] (H1-\node) edge (H2-\node);
		
\end{tikzpicture} 
\caption{Visualization of the unbounded width case for $d_0 = d = 3, d_1 = d_2, d_3 = 1$: Unbounded width, i.e.\ $ d_1 = d_2 \rightarrow \infty$, but otherwise fixed setup.}
\label{fig:visualization_NN_DK_inf_width}
\end{figure}
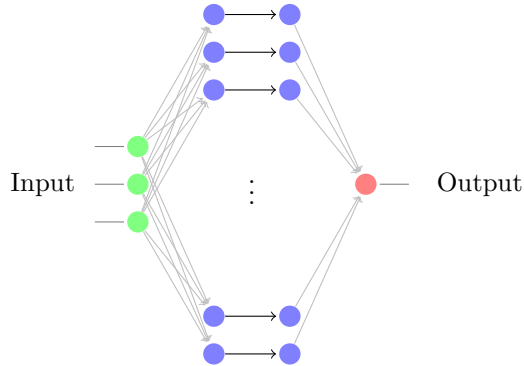

We focus again on a univariate output, i.e.\ $d_3 = 1$. 
We collect the requirements in the following:
\begin{assumption}%
[Universal approximation in the width of the network]
\label{ass:wide_assumptions} ~ 
\begin{enumerate}
\item We consider a single-dimensional kernel $k_2$ that is built using a radial basis function $\varphi: \R_{\geq 0} \rightarrow 
\R$ such that 
\begin{equation*}
\overline{\mathrm{span} \{ f ~|~ \exists a > 0 \forall x > 0:  f(x) = 
\varphi(ax) \} }= C([0, 1]).
\end{equation*}

\item At least 2 different centers $z_1 \neq z_2$ are given.
\end{enumerate}
\end{assumption}

\begin{rem}
An exact characterization of the first requirement within \Cref{ass:wide_assumptions} seems to be possible with the help of advanced tools from harmonic analysis, 
however this is beyond the scope of this paper. 
We remark, that e.g.\ the Gaussian or the basic Mat\'ern kernel satisfy this assumption due to the Stone--Weierstrass theorem.
\end{rem}

Then, we obtain the following approximation statement: 

\begin{theorem}[Universal approximation for unbounded width] \label{th:unbounded_width}
For any $d \in \N$, 
let $\Omega = [0,1]^{d}$ and consider an arbitrary continuous function $f: \Omega \rightarrow \R$. Then it is possible under the Assumptions \ref{ass:wide_assumptions} to approximate this function $f$ to arbitrary accuracy using an SDKN of depth $L=1$ and $M=2$ centers, i.e.\
\begin{align*}
  \lim_{d_1=d_2 \rightarrow \infty} \dist(f, \mathcal{F}_{1, d_1, 2}(\Omega)) = 0.
\end{align*}
\end{theorem}

The proof of Theorem \ref{th:unbounded_width} can be found in the appendix, see \proofref{th:unbounded_width}. %

\subsection{Universal approximation in the depth of the network} 
\label{subsubsec:infinite_depth_case}

For standard feedforward neural networks, it was pointed out that \textit{deep} networks asymptotically perform better than \textit{shallow} networks: Most of those work focussed on the \relu \ activation function, see e.g.\ \cite{guhring2020error} or \cite{elbrachter2021deep} for a discussion and an overview. \\
In this section we prove the universal approximation property also for the deep case of the SDKN, by mimicing common NN approximation constructions. By using the flexibility of kernel methods, we establish these results for a wide range of kernels. \\
For this, we discuss in Section \ref{subsubsec:flat_limits} the so called flat limit of kernels. Subsequently, we provide in Section \ref{subsubseq:deep_construction_of_approximant} the approximation results for SDKNs with unbounded depth.

\begin{figure}[H]
\def\layersep{2cm}
\centering %
\begin{tikzpicture}[scale=0.45, shorten >=1pt,->,draw=black!50, node distance=\layersep]
    \tikzstyle{every pin edge}=[<-,shorten <=1pt]
    \tikzstyle{neuron}=[circle,fill=black!25,minimum size=8pt,inner sep=0pt]
    \tikzstyle{input neuron}=[neuron, fill=green!50];
    \tikzstyle{output neuron}=[neuron, fill=red!50];
    \tikzstyle{hidden neuron}=[neuron, fill=blue!50];
    \tikzstyle{annot} = [text width=4em, text centered]

	\node[annot] () at (-1.5*\layersep, -2.5) {Input};
	\node[annot] () at (8*\layersep, -2.5) {Output};    

    \foreach \name / \y in {1,2,3}
        \node[input neuron,pin={[pin edge={-}]left:}] (I-\name) at (0,-.5-\y) {};

    \foreach \name / \y in {1,...,4}
        \node[hidden neuron] (H1-\name) at (\layersep,-\y) {};
        
    \foreach \name / \y in {1,...,4}
        \node[hidden neuron] (H2-\name) at (2*\layersep,-\y) {};        
        
    \foreach \name / \y in {1,...,4}
        \node[hidden neuron] (H3-\name) at (3*\layersep,-\y) {};
        
    \foreach \name / \y in {1,...,4}
        \node[hidden neuron] (H4-\name) at (4.5*\layersep,-\y) {};
        
    \foreach \name / \y in {1,...,4}
        \node[hidden neuron] (H5-\name) at (5.5*\layersep,-\y) {};

    \foreach \name / \y in {1}
        \node[output neuron,pin={[pin edge={-}]right:}] (O-\name) at (6.5*\layersep,-1.5-\y) {};    

    \node(dots) at (3.75*\layersep,-2.5){\dots};

    \foreach \source in {1,...,3}
	    \foreach \dest in {1,...,4}
	        \path [lightgray] (I-\source) edge (H1-\dest);    
    
    \foreach \source in {1,...,4}
        \foreach \dest in {1,...,4}
            \path [lightgray] (H2-\source) edge (H3-\dest);
            
	\foreach \source in {1,...,4}
		\foreach \dest in {1}
	        \path [lightgray] (H5-\source) edge (O-\dest);   

	\foreach \node in {1,...,4}
		\path [black] (H1-\node) edge (H2-\node);
		
	\foreach \node in {1,...,4}
		\path [black] (H4-\node) edge (H5-\node);

\end{tikzpicture}
\caption{Visualization of the unbounded depth case for $d_0 = 3, d_1 = .. = d_{L-1} = 4, d_L = 1$: Unbounded depth, i.e.\ $L \in \N$ arbitrary, but otherwise bounded width and number of centers.}
\label{fig:visualization_NN_DK_inf_depth}
\end{figure}
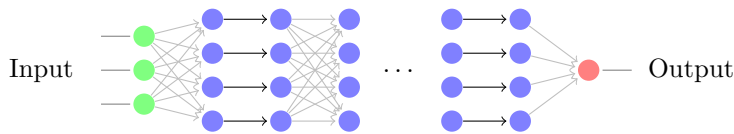

\subsubsection{Flat limit of kernels}
\label{subsubsec:flat_limits}

As elaborated in Section \ref{subsec:approx_with_kernels}, 
the radial basis function kernels depend on a shape parameter $\varepsilon > 0$. 
The case of $\varepsilon \rightarrow 0$ is refered to as the \textit{flat limit} of kernels, 
as the shape of the RBF functions becomes very flat - in contrast to their peaky shape for large values of $\varepsilon$. 
This case was studied in the kernel literature \cite{DRISCOLL2002413, LARSSON2005103}, and close connections to polynomial interpolation were derived.
Especially in the $1D$ case $\Omega = [0, 1]$ 
- which is present in the activation function layers due to the single-dimensional kernels - precise statements can be derived, 
that hold under mild requirements. 
The following theorem gives a precise statement for the use of two and three interpolation points (see \cite[Theorem 3.1]{DRISCOLL2002413}).

\begin{theorem}
\label{th:flat_kernel_theorem}
Let $N \in \{2, 3\}$ distinct data points $X_N = \{x_1, \ldots, x_N\} \subset \R$ and corresponding target values $\{f_1, \ldots, f_N\} \subset \R$ be given. 
Suppose the basis function
\begin{align*}
\varphi(r) = \sum_{j=0}^\infty a_{j} r^{2j} %
\end{align*}
is strictly positive definite (i.e.\ the kernel $k(x,z) := \varphi(\varepsilon\Vert x - z \Vert)$ is strictly positive definite). If 
\begin{align}
\label{eq:requirements_flat_kernel}
\begin{aligned}
a_0 &\neq 0,~ a_1 \neq 0 \hspace{1.7cm} \text{in the case $N=2$} \\
a_1 &\neq 0,~ 6a_0 a_2 - a_1^2 \neq 0 \hspace{.4cm} \text{in the case $N=3$},
\end{aligned}
\end{align}
then the kernel interpolant $s(x, \varepsilon)$ based on the nodes $X_N$ satisfies for any $x \in \R$
\begin{align}
\label{eq:convergence_to_polynomial}
\lim_{\varepsilon \rightarrow 0} s(x, \varepsilon) = p_{N-1}(x),
\end{align}
where $p_{N-1}(x) \in \mathbb{P}_{N-1}$ is the polynomial interpolating $f$ on the nodes $X_N$. %
\end{theorem}

We remark that the convergence in Eq.\ \eqref{eq:convergence_to_polynomial} is not only pointwise but even a compact convergence, i.e.\ uniform convergence on 
any compact subset $K \subset \R$. Like this, based on two distinct centers, it is possible to approximate any affine mapping $x \mapsto ax + b$ by driving the 
kernel parameter $\varepsilon$ to zero. Based on three pairwise distinct centers, it is even possible to approximate arbitrary quadratic polynomials $x \mapsto 
ax^2 + bx + c$.

A modification of the kernel width parameter of the kernel used in the SDKN setup is not required to achieve this flat limit: 
Any single-dimensional kernel function layer within the SDKN is preceded by a linear layer, 
thus the case of a small kernel parameter $\varepsilon$ can be realized by decreasing the magnitude of the preceding linear layer, 
i.e. $A \rightarrow \varepsilon A$, see \Cref{lem:basic_operations}.

\subsubsection{Deep construction of approximant}
\label{subsubseq:deep_construction_of_approximant}

In this section we provide a construction that shows that a deep SDKN can approximate any continuous function on a compact domain. 
The corresponding statement is made precise in Theorem \ref{th:infinite_depth}. 
The construction is inspired from related constructions in e.g.\ \cite{elbrachter2021deep, YAROTSKY2017103}, 
but a substantially smaller layout (in terms of the depth) is possible by using kernel properties. 
This and further comments are given in Remark \ref{rem:deep_case_construction}. \\

We will focus on $\Omega \subset \R^d_{\geq 0}$ for simplicity. We will refer to the following common assumptions on the SDKN
under consideration:
\begin{assumption}%
[Universal approximation in the depth of the network]
\label{ass:deep_assumptions} ~
\begin{enumerate}%
\item The kernels of the single-dimensional kernel layers satisfy the requirements within Eq.~\eqref{eq:requirements_flat_kernel} of \Cref{th:flat_kernel_theorem}.
\item There are given $3$ centers $z_1, z_2, z_3 \in \R_{\geq 0}^d$ such that $z_1^{(j)}, z_2^{(j)}, z_3^{(j)}$ are pairwise distinct for $j=1, \ldots, d$.
\end{enumerate}
\end{assumption}
\begin{rem}
There are versions of Theorem~\ref{th:flat_kernel_theorem} that result into the approximation of polynomials of higher degree, thus potentially simplifying 
some of the constructions that follow. However, this is possible only at the price of imposing stricter conditions on $X_N$, which may hold for the input 
centers, but not necessarily for the propagated ones. 
\end{rem}

Now we are in a position to state some preliminary constructions that will be used in the following. These constructions are similar to common constructions for 
ReLU-networks.

\begin{lemma}[Identity and squaring operation]
\label{lem:basic_operations}
Let $\Omega \subset \R_{\geq 0}$ be a compact interval. The functions $\psi_1: \Omega \rightarrow \R, x \mapsto x$ and $\psi_2: \Omega \rightarrow \R, x \mapsto x^2$ can be approximated by an SDKN satisfying the Assumptions \ref{ass:deep_assumptions} to arbitrary accuracy:
\begin{align*}
\dist(\psi_1, ~ \mathcal{F}_{1,1,3}(\Omega)) &= 0, ~~ \psi_1(x) = x \\
\dist(\psi_2, ~ \mathcal{F}_{1,1,3}(\Omega)) &= 0, ~~ \psi_2(x) = x^2.
\end{align*}
\end{lemma}

\begin{proof}
Since both the mappings $\psi_1, \psi_2$ are $\R \rightarrow \R$, we choose $d_0 = d_1 = d_2 = 1$. The linear mapping in the beginning acts only as a scaling $x \rightarrow \varepsilon \cdot x$, the linear mapping in the end is not required, i.e.\ its weight can be set to $1$. Thus the output of the SDKN is given by
\begin{align*}
f_\varepsilon(x) = \sum_{i=1}^3 \alpha_i k(\varepsilon x, \varepsilon z_i)
\end{align*}
with $\alpha_i \in \R, \varepsilon > 0$ for a kernel $k$ satisfying the requirements within Theorem \ref{th:flat_kernel_theorem}. We choose $\alpha_i$ such that we have $f_\sigma(z_i) = \psi_j(z_i)$ for $i=1, 2$ ($j=1$) respective $i=1,2,3$ ($j=2$).
These are standard interpolation conditions that give the linear equation system (in case of $j=2$)
\begin{align*}
\begin{pmatrix}
k(\varepsilon z_1, \varepsilon z_1) 	& k(\varepsilon z_1, \varepsilon z_2)		& k(\varepsilon z_1, \varepsilon z_3) \\
k(\varepsilon z_2, \varepsilon z_1) 	& k(\varepsilon z_2, \varepsilon z_2)		& k(\varepsilon z_2, \varepsilon z_3) \\
k(\varepsilon z_3, \varepsilon z_1) 	& k(\varepsilon z_3, \varepsilon z_2)		& k(\varepsilon z_3, \varepsilon z_3) \\
\end{pmatrix}
\begin{pmatrix}
\alpha_1 \\
\alpha_2 \\
\alpha_3 \\
\end{pmatrix}
= 
\begin{pmatrix}
\psi_2(z_1) \\
\psi_2(z_2) \\
\psi_2(z_3) \\
\end{pmatrix}.
\end{align*}
For $\varepsilon \rightarrow 0$ the flat limit is encountered: 
Leveraging Theorem \ref{th:flat_kernel_theorem} yields 
\begin{align*}
\lim_{\varepsilon \rightarrow 0} f_\varepsilon(x) = p_2(x),
\end{align*}
with $p_2(x)$ being the the interpolating polynomial for the data $(z_1, \psi_j(z_1))$, $(z_2, \psi_j(z_2))$, $(z_3, \psi_j(z_3))$,
which is given by $p_2(x) = x$ for $j=1$ or $p_2(x) = x^2$ for $j=2$.
\end{proof}

We remark that the mapped centers $\psi_j(z_1), \psi_j(z_2), \psi_j(z_3)$ after the identity or squaring operation are still pairwise different due to the injectivity of $\psi_1, \psi_2$ on $\Omega \subset \R_{\geq 0}$, 
i.e.\ we have that $\psi_j(z_1)$, $\psi_j(z_2)$, $\psi_j(z_3)$ are pairwise distinct for both $j=1$ and $j=2$.

\begin{lemma}[Depth adjustment]
\label{lem:depth_adjustment}
Let $\Omega \subset \R_{\geq 0}$ be a compact interval and assume \Cref{ass:deep_assumptions}.
Given an SDKN of depth $L_1$, it is possible to approximate this SDKN to arbitrary accuracy using another SDKN of arbitrary depth $L_2 > L_1$, i.e., 
\begin{align*}
\forall w \in \N ~ \forall f \in \mathcal{F}_{L_1, w, 3}(\Omega): ~ \dist(f, \mathcal{F}_{L_2, w, 3}(\Omega)) = 0.
\end{align*}
\end{lemma}

\begin{proof}
It suffices to prove the existence of a depth $L_2 - L_1$ SDKN that realizes an arbitrary accurate approximation $\widetilde{\id}$ of the identity mapping $\id: \R^{d_0} \rightarrow \R^{d_0}$:
If this is the case, we have $\mathcal{F}_{L_1, w, 3} \ni f = f \circ \id \approx f \circ \widetilde{\id} \in \mathcal{F}_{L_2, w, 3}$. \\
The realization of such an approximation $\tilde{\id}$ is possible due to Lemma \ref{lem:basic_operations} which states that $\psi_1(x) = x$ can be approximated to arbitrary accuracy. This can be employed in every input dimension $L_2 - L_1$ times.
\end{proof}

\begin{lemma}[Product module]
\label{lem:product_module}
Let $\Omega \subset \R_{\geq 0}^2$ be a compact domain. The function $\psi: \Omega \rightarrow \R, (x, y) \mapsto xy$ can be approximated by an SDKN satisfying 
the Assumptions \ref{ass:deep_assumptions} to arbitrary accuracy, i.e., 
\begin{align*}
\dist(\psi(x, y) = xy, ~ \mathcal{F}_{1,3,3}(\Omega)) = 0
\end{align*}
\end{lemma}

\begin{proof}
In the case that the input data $x$ and $y$ are linearly dependent (that means: the vectors $(x_i)_{i=1}^N \in \R^N$ (using the first dimension of the training data $X_N$) and $(y_i)_{i=1}^N \in \R^N$ (using the second dimension of the training data)), 
the output $xy = c \cdot x^2$ for some $c \in \R$ can be approximated by applying the squaring operation $x \mapsto x^2$ of \Cref{lem:basic_operations} to the input $x$ with a proper scaling. \\
Thus assume linear independence in the input. 
The \textit{product module} for this case is depicted in Figure \ref{fig:product_module}. We make use of
\begin{align} \label{eq:multiplication}
xy = \frac{1}{2\beta} \left( (x+\beta y)^2 - x^2 - \beta^2 y^2 \right), \beta > 0.
\end{align}
The linear combination $(x,y) \rightarrow x + \beta y$ can be computed by the linear layer in the beginning, which is possible based on \Cref{prop:linear_layer} due to the linear independence of the inputs. The linear combination of the squares is possible for the same reason.
So far the parameter $\beta$ can be chosen arbitrarily as long as $\beta \neq 0$.
Thus we can make a special choice, such that the assumptions for the application of \Cref{lem:basic_operations} are satisfied: 
It is required that $z_i^{(1)} + \beta z_i^{(2)}$, $i=1, \ldots, 3$ are pairwise distinct, which can be enforced based on $\beta \neq 0$ and the pairwise 
distinctness of $z_1^{(j)}, z_2^{(j)}, z_3^{(j)}$ for $j=1,2$. %
\end{proof}

In the following we restrain from explicitly elaborating on the realization of the linear layers via \Cref{prop:linear_layer}. Even in case of linear dependency, all required linear layers can be realized as proven in \Cref{prop:linear_layer_2}.

\begin{figure}[H]
\def\layersep{2cm}
\centering %
\begin{tikzpicture}[scale=0.45, shorten >=1pt,->,draw=black!50, node distance=\layersep]
    \tikzstyle{every pin edge}=[<-,shorten <=1pt]
    \tikzstyle{neuron}=[circle,fill=black!25,minimum size=8pt,inner sep=0pt]
    \tikzstyle{input neuron}=[neuron, fill=green!50];
    \tikzstyle{output neuron}=[neuron, fill=red!50];
    \tikzstyle{hidden neuron}=[neuron, fill=blue!50];
    \tikzstyle{relu neuron}=[neuron, fill=black, minimum size=5pt];    
    \tikzstyle{annot} = [text width=4em, text centered]
    \tikzstyle{module}=[circle,fill=black!25,minimum size=8pt,inner sep=0pt]

    \node[input neuron, pin={[pin edge={-}]left:}, label={\tiny{$x$}}] (I-1) at (0,-1) {};
    \node[input neuron, pin={[pin edge={-}]left:}, label={\tiny{$y$}}] (I-2) at (0,-3) {};        

	\node[hidden neuron, label={\tiny{$x$}}] (H1-1) at (\layersep,0) {};
    \node[hidden neuron, label={\tiny{$y$}}] (H1-2) at (\layersep,-2) {};
	\node[hidden neuron, label={\tiny{$x+\beta y$}}] (H1-3) at (\layersep,-4) {};

	\node[hidden neuron, label={\tiny{$x^2$}}] (H2-1) at (3*\layersep,0) {};
	\node[hidden neuron, label={\tiny{$y^2$}}] (H2-2) at (3*\layersep,-2) {};
    \node[hidden neuron, label={\tiny{$(x+\beta y)^2$}}] (H2-3) at (3*\layersep,-4) {};

    \node[output neuron, pin={[pin edge={-}]right:}, label={\tiny{$xy$}}] (O-1) at (4*\layersep,-2) {};    

	\path [lightgray] (I-1) edge (H1-1);
	\path [lightgray] (I-2) edge (H1-2);
	\path [lightgray] (I-1) edge (H1-3);
	\path [lightgray] (I-2) edge (H1-3);
    
	\foreach \source in {1,...,3}
		\foreach \dest in {1}
	        \path [lightgray] (H2-\source) edge (O-\dest);   

	\foreach \node in {1,...,3}
		\path [black] (H1-\node) edge (H2-\node);

	\node[annot] () at (6*\layersep, -2.1) {\Huge{=}};    

    \node[input neuron, pin={[pin edge={-}]left:}, label={\tiny{$x$}}] (I2-1) at (8*\layersep,-1) {};
    \node[input neuron, pin={[pin edge={-}]left:}, label={\tiny{$y$}}] (I2-2) at (8*\layersep,-3) {};        
	\node[module] (P-1) at (9*\layersep, -2) {P};  		
    \node[output neuron, pin={[pin edge={-}]right:}, label={\tiny{$x y$}}] (O-1) at (10*\layersep, -2) {};    

	\path (I2-1) edge (P-1);
	\path (I2-2) edge (P-1);	
	
	\path (P-1) edge (O-1);

\end{tikzpicture}
\caption{Visualization of the product module (left) with corresponding abbreviation (right), which approximates the function $\psi(x, y) = xy$ based on Eq.\ \eqref{eq:multiplication} to arbitrary accuracy. The weights of the final linear layer of the left figure are each $\frac{1}{2\beta}$.}
\label{fig:product_module}
\end{figure}
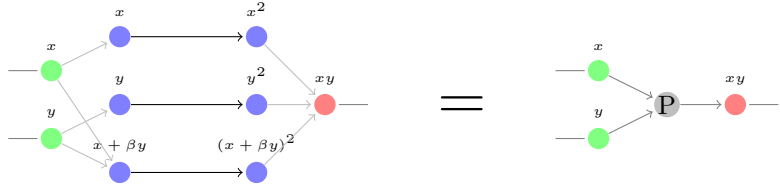

\begin{lemma}[Univariate monomial module]
\label{lem:univariate_monomial_module}
Let $\Omega \subset \R_{\geq 0}$ be a compact interval. The function $\psi: \Omega \rightarrow \R, x \mapsto x^n$ with $n \in \N$ can be approximated by an SDKN from $\mathcal{F}_{L, 4, 3}$ satisfying the Assumptions \ref{ass:deep_assumptions} to arbitrary accuracy:
\begin{align*}
\dist(\psi(x) = x^n, ~ \mathcal{F}_{L,4,3}(\Omega)) = 0
\end{align*}
with depth $L = \max( \lceil \log_2(n) \rceil, 1)$.
\end{lemma}

\begin{proof}
The univariate monomial module is depicted in \Cref{fig:univariate_monomial_module}.
The case $n=1$ is simply the identity, 
whereas $n=2$ corresponds to the squaring operation, 
which were both treated in \Cref{lem:basic_operations}. 
Thus we focus on $n \geq 3$:
The approximation of $x^{2^{\lfloor \log_2(n) \rfloor}}$ can be realized by concatenating the squaring operation from \Cref{lem:basic_operations} $\lfloor \log_2(n) \rfloor$ times. 
The remaining factor $x^{n - 2^{\lfloor \log_2(n) \rfloor}}$ can be multiplied (in case $n - 2^{\lfloor \log_2(n) \rfloor} \neq 0$) in the end by the multiplication operation: 
The factor $x^{n - 2^{\lfloor \log_2(n) \rfloor}}$ was built in parallel and collected in the first dimension 
(i.e.\ in the top of Figure \ref{fig:univariate_monomial_module}) %
\end{proof}

The assumption on the pairwise distinctness of the centers transfers through the whole monomial module, as the approximation of the squaring and the product module is exact in the respective centers, and the pairwise distinctness of $z_1, z_2, z_3$ implies the pairwise distinctness of $z_1^n, z_2^n, z_3^n$ due to $\Omega \subset \R_{\geq 0}$.

\begin{figure}[H]
\def\layersep{2cm}
\centering %
\begin{tikzpicture}[scale=0.45, shorten >=1pt,->,draw=black!50, node distance=\layersep]
    \tikzstyle{every pin edge}=[<-,shorten <=1pt]
    \tikzstyle{neuron}=[circle,fill=black!25,minimum size=8pt,inner sep=0pt]
    \tikzstyle{input neuron}=[neuron, fill=green!50];
    \tikzstyle{output neuron}=[neuron, fill=red!50];
    \tikzstyle{hidden neuron}=[neuron, fill=blue!50];
    \tikzstyle{relu neuron}=[neuron, fill=black, minimum size=5pt];    
    \tikzstyle{annot} = [text width=4em, text centered]
    \tikzstyle{module}=[circle,fill=black!25,minimum size=8pt,inner sep=0pt]
    \tikzstyle{unused module}=[circle,fill=black!15,minimum size=8pt,inner sep=0pt]    

	\node[input neuron, pin={[pin edge={-}]left:}, label={\tiny{$x$}}] (I-1) at (0,-2) {};

	\node[hidden neuron, label={\tiny{$x$}}] (H1-1) at (\layersep,0) {};
    \node[hidden neuron, label={\tiny{$x^2$}}] (H1-2) at (\layersep,-2) {};    
	\node[module] (P1) at (1.8*\layersep, -1) {P};

	\node[hidden neuron, label={\tiny{$x^3$}}] (H2-1) at (3*\layersep,0) {};
    \node[hidden neuron, label={\tiny{$x^4$}}] (H2-2) at (3*\layersep,-2) {};    
  	\node[unused module] (P2) at (3.8*\layersep, -1) {P};    
    
	\node[hidden neuron, label={\tiny{$x^3$}}] (H3-1) at (5*\layersep,0) {};
    \node[hidden neuron, label={\tiny{$x^8$}}] (H3-2) at (5*\layersep,-2) {};    
  	\node[module] (P3) at (5.8*\layersep, -1) {P};        

    \node[output neuron, pin={[pin edge={-}]right:}, label={\tiny{$x^{11}$}}] (O-1) at (7*\layersep, 0) {};    

	\path [lightgray] (I-1) edge (H1-1);
	\path [lightgray] (I-1) edge (H1-2);
	
	\path (H1-2) edge (H2-2);
	\path (H1-1) edge (P1);
	\path (H1-2) edge (P1);		
	
	\path (P1) edge [out=70, in=180] (H2-1);	
	\path (H2-2) edge (H3-2);
	
	\draw[->, lightgray, dashed] (H2-1) -- (P2);	
	\draw[->, lightgray, dashed] (H2-2) -- (P2);	

	\draw[->, lightgray, dashed] (P2) to[out=70, in=200] (H3-1);		
	\path (H2-1) edge (H3-1);
	\path (H3-1) edge (P3);
	\path (H3-2) edge (P3);
	
	\path (P3) edge [out=70, in=180] (O-1);

	\node[annot] () at (8.5*\layersep, -1.1) {\Huge{=}};    

    \node[input neuron, pin={[pin edge={-}]left:}, label={\tiny{$x$}}] (I2-1) at (10*\layersep,-1) {};
	\node[module] (M-1) at (11*\layersep, -1) {M};
    \node[output neuron, pin={[pin edge={-}]right:}, label={\tiny{$x^{11}$}}] (O-1) at (12*\layersep, -1) {};

	\path (I2-1) edge (M-1);
	
	\path (M-1) edge (O-1);

\end{tikzpicture}
\caption{Visualization of the monomial module (left) with corresponding abbreviation (right), 
which approximates the function $\psi(x) = x^n$ to arbitrary accuracy. Here, we have $n=11$. 
The grey circles ``P'' refer to the product module of \Cref{lem:product_module}. 
The dashed lines indicate the position of another potential product module, which is however not required as there is no $x^4$ contribution to build $x^{11}$.}
\label{fig:univariate_monomial_module}
\end{figure}
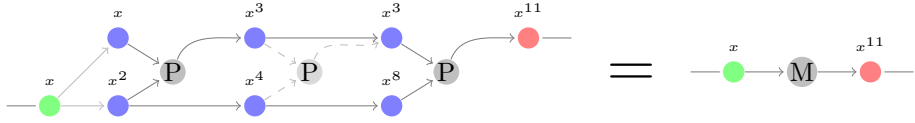

The following Lemma \ref{lem:bivariate_monomial_module} shows that the approximation of polynomials in two inputs is possible by combining the monomial and the 
product module.

\begin{lemma}[Bivariate monomial module]
\label{lem:bivariate_monomial_module}
Let $\Omega \subset \R_{\geq 0}^2$ be a compact domain. The function $\psi: \Omega \rightarrow \R, (x, y) \mapsto x^{a}y^{b}$ with $a, b \in \N$ can be approximated by an SDKN from $\mathcal{F}_{L, 8, 3}(\Omega)$ satisfying the Assumptions \ref{ass:deep_assumptions} to arbitrary accuracy:
\begin{align*}
\dist(\psi(x,y) = x^{a}y^{b}, \mathcal{F}_{L,8,3}(\Omega)) = 0
\end{align*}
with depth $L = \lceil \log_2(\max(a, b)) \rceil + 1$.
Furthermore the following extension holds:
\begin{align*}
\dist(\tilde{\psi}(x,y,z) = x^{a}y^{b}+\beta z, \mathcal{F}_{L,9,3}(\Omega)) = 0, ~ \beta \in \R.
\end{align*}
\end{lemma}

\begin{proof}
The standard bivariate monomial module for the approximation of $\psi(x,y)$ is depicted in Figure \ref{fig:bivariate_monomial_module}, top: Based on the two inputs $x$ and $y$, the univariate monomial module is applied to each of them. In case of $a \neq b$, the depth of the univariate monomial modules is adjusted to $\max (a, b)$ using Lemma \ref{lem:depth_adjustment}. This requires a depth of $L_1 = \max(\lceil \log_2(\max(a, b)) \rceil, 1)$. Subsequently, the product module from Lemma \ref{lem:product_module} is used to compute the final approximation of $x^{a}y^{b}$. Thus the total depth is given by $L = L_1 + 1$. \\
For the proof of the extension, i.e.\ the approximation of $\psi(x,y,z)$, the same construction can be extended with a final linear layer that adds the $+\beta z$ term. This is depicted in Figure \ref{fig:bivariate_monomial_module}, bottom.
\end{proof}

It might happen that the centers collapse after the standard bivariate monomial module, i.e.\ that $\psi(z_1^{(1)}, z_1^{(2)})$, $\psi(z_2^{(1)}, z_2^{(2)})$, $\psi(z_3^{(1)}, z_3^{(2)})$ are no longer pairwise distinct. Therefore the extension to $\tilde{\psi}$ was introduced. As $\beta \in \R$ can be chosen arbitrarily, this alleviates the beforementioned center collapse.

\begin{figure}[H]
\def\layersep{2cm}
\centering %
\begin{tikzpicture}[scale=0.45, shorten >=1pt,->,draw=black!50, node distance=\layersep]
    \tikzstyle{every pin edge}=[<-,shorten <=1pt]
    \tikzstyle{neuron}=[circle,fill=black!25,minimum size=8pt,inner sep=0pt]
    \tikzstyle{input neuron}=[neuron, fill=green!50];
    \tikzstyle{output neuron}=[neuron, fill=red!50];
    \tikzstyle{hidden neuron}=[neuron, fill=blue!50];
    \tikzstyle{relu neuron}=[neuron, fill=black, minimum size=5pt];    
    \tikzstyle{annot} = [text width=4em, text centered]
    \tikzstyle{module}=[circle,fill=black!25,minimum size=8pt,inner sep=0pt]

	\node[input neuron, pin={[pin edge={-}]left:}, label={\tiny{$x$}}] (I-1) at (0,0) {};
	\node[input neuron, pin={[pin edge={-}]left:}, label={\tiny{$y$}}] (I-2) at (0,-2) {};

	\node[module] (M-1) at (\layersep,0) {M};
    \node[module] (M-2) at (\layersep,-2) {M};    
    
    \node[hidden neuron, label={\tiny{$x^{a}$}}] (H1-1) at (2*\layersep, 0) {};    
    \node[hidden neuron, label={\tiny{$y^{b}$}}] (H1-2) at (2*\layersep, -2) {};        
    
	\node[module] (P-1) at (3*\layersep, -1) {P};    

    \node[output neuron, pin={[pin edge={-}]right:}, label={\tiny{$x^{a}y^{b}$}}] (O-1) at (4*\layersep, -1) {};    

	\path (I-1) edge (M-1);
	\path (I-2) edge (M-2);
	
	\path (M-1) edge (H1-1);
	\path (M-2) edge (H1-2);
	
	\path (H1-1) edge (P-1);
	\path (H1-2) edge (P-1);
	
	\path (P-1) edge (O-1);

	\node[annot] () at (5.5*\layersep, -1.1) {\Huge{=}};    

    \node[input neuron, pin={[pin edge={-}]left:}, label={\tiny{$x$}}] (I2-1) at (7*\layersep,0) {};
    \node[input neuron, pin={[pin edge={-}]left:}, label={\tiny{$y$}}] (I2-2) at (7*\layersep,-2) {};        
	\node[module] (B-1) at (8*\layersep, -1) {B};  		
    \node[output neuron, pin={[pin edge={-}]right:}, label={\tiny{$x^{a}y^{b}$}}] (O-1) at (9*\layersep, -1) {};    

	\path (I2-1) edge (B-1);
	\path (I2-2) edge (B-1);	
	
	\path (B-1) edge (O-1);

\end{tikzpicture}
\rule{.9\textwidth}{.2pt}
\begin{tikzpicture}[scale=0.45, shorten >=1pt,->,draw=black!50, node distance=\layersep]
    \tikzstyle{every pin edge}=[<-,shorten <=1pt]
    \tikzstyle{neuron}=[circle,fill=black!25,minimum size=8pt,inner sep=0pt]
    \tikzstyle{input neuron}=[neuron, fill=green!50];
    \tikzstyle{output neuron}=[neuron, fill=red!50];
    \tikzstyle{hidden neuron}=[neuron, fill=blue!50];
    \tikzstyle{relu neuron}=[neuron, fill=black, minimum size=5pt];    
    \tikzstyle{annot} = [text width=4em, text centered]
    \tikzstyle{module}=[circle,fill=black!25,minimum size=8pt,inner sep=0pt]

	\node[input neuron, pin={[pin edge={-}]left:}, label={\tiny{$x$}}] (I-1) at (0,0) {};
	\node[input neuron, pin={[pin edge={-}]left:}, label={\tiny{$y$}}] (I-2) at (0,-2) {};
	\node[input neuron, pin={[pin edge={-}]left:}, label={\tiny{$z$}}] (I-3) at (0,-4) {};	

	\node[module] (M-1) at (\layersep,0) {M};
    \node[module] (M-2) at (\layersep,-2) {M};
    
    \node[hidden neuron, label={\tiny{$x^{a}$}}] (H1-1) at (2*\layersep, 0) {};    
    \node[hidden neuron, label={\tiny{$y^{b}$}}] (H1-2) at (2*\layersep, -2) {};        
    \node[hidden neuron, label={\tiny{$z$}}] (H1-3) at (2*\layersep, -4) {};        
    
	\node[module] (P-1) at (3*\layersep, -1) {P};    
	
    \node[hidden neuron, label={\tiny{$x^{a}y^{b}$}}] (H2-1) at (4*\layersep, -1) {};    
    \node[hidden neuron, label={\tiny{$z$}}] (H2-2) at (4*\layersep, -3) {};        

    \node[output neuron, pin={[pin edge={-}]right:}, label={\tiny{$x^{a}y^{b}+\beta z$}}] (O-1) at (5*\layersep, -2) {};    

	\path (I-1) edge (M-1);
	\path (I-2) edge (M-2);
	\path (I-3) edge (H1-3);
	
	\path (M-1) edge (H1-1);
	\path (M-2) edge (H1-2);
	
	\path (H1-1) edge (P-1);
	\path (H1-2) edge (P-1);
	\path (H1-3) edge (H2-2);
	
	\path (P-1) edge (H2-1);

	\path (H2-1) edge (O-1);
	\path (H2-2) edge (O-1);

	\node[annot] () at (6.5*\layersep, -2.1) {\Huge{=}};    

    \node[input neuron, pin={[pin edge={-}]left:}, label={\tiny{$x$}}] (I2-1) at (8*\layersep,0) {};
    \node[input neuron, pin={[pin edge={-}]left:}, label={\tiny{$y$}}] (I2-2) at (8*\layersep,-2) {};        
    \node[input neuron, pin={[pin edge={-}]left:}, label={\tiny{$z$}}] (I2-3) at (8*\layersep,-4) {};            
	\node[module] (B-1) at (9*\layersep, -2) {B'};  		
    \node[output neuron, pin={[pin edge={-}]right:}, label={\tiny{$x^{a}y^{b}+\beta z$}}] (O-1) at (10*\layersep, -2) {};    

	\path (I2-1) edge (B-1);
	\path (I2-2) edge (B-1);	
	\path (I2-3) edge (B-1);	
	
	\path (B-1) edge (O-1);

\end{tikzpicture}
\caption{Top: Visualization of the bivariate monomial module (left) with corresponding abbreviation (right), which approximates the function $\psi(x, y) = x^{a}y^{b}$ to arbitrary accuracy. The grey circles ``M'' refer to the univariate monomial module, whereas the grey circle ``P'' refers to the product module. \newline
Bottom: Visualization for the extension, i.e.\ the approximation of $\tilde{\psi}$.}
\label{fig:bivariate_monomial_module}
\end{figure}
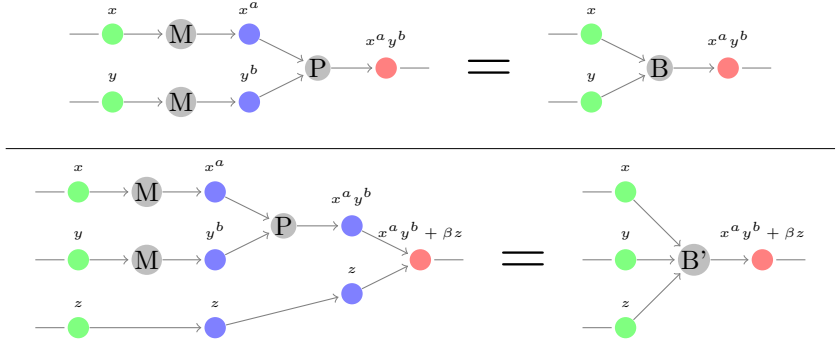

Using these preliminary construction results, we can finally show how to add successively arbitrary polynomials to build a sum:

\begin{lemma}[Addition module]
\label{lem:addition_module}
Let $\Omega \subset \R_{\geq 0}^d \times \R$ be a compact domain, $\alpha, \beta \in \R$, $n_j \in \N$, and define the mapping
\begin{align*}
\psi: \Omega &\rightarrow \R^{d+1}, \begin{pmatrix} x^{(1)} \\ \vdots \\ x^{(d)} \\ S \end{pmatrix} \mapsto \begin{pmatrix} x^{(1)} \\ \vdots \\ x^{(d)} \\ S' \end{pmatrix}, \\
&\text{with} ~ S' := S+\alpha \cdot \prod_{j=1}^d (x^{(j)})^{n_j} + \beta x^{(d)}.
\end{align*}
Then, under Assumptions \ref{ass:deep_assumptions}, an SDKN can approximately realize the mapping $\psi$ to arbitrary accuracy, i.e., 
\begin{align*}
\dist(\psi(x), \mathcal{F}_{L,d+8,3}(\Omega)) = 0, \qquad L \leq d \cdot \lceil \log_2(\max_{j=1,\ldots,d} n_j + 1) \rceil.
\end{align*}
\end{lemma}

\begin{rem}
\label{rem:additional_beta_term}
The additional summand $\beta x^{(d)}$ is included for the same reason why we included the parameter $\beta$ in the proof of Lemma \ref{lem:product_module}: 
It might happen, that the used centers $z_1, z_2, z_3$ coincide after the mapping $\psi$, 
i.e.\ $\psi^{(d+1)}(z_1),$ $\psi^{(d+1)}(z_2),$ $\psi^{(d+1)}(z_3)$ are no longer pairwise distinct. 
By including the $\beta$-summand, we can always find a $\beta \in \R$ such that the components $\psi(z_1)^{(j)}$, $\psi(z_2)^{(j)}$, $\psi(z_3)^{(j)}$, $j=1, 
\ldots, d$ of the mapped centers $\psi(z_1)$, $\psi(z_2)$, $\psi(z_3)$ are pairwise distinct also for $j = d+1$.
\end{rem}

\begin{proof}[Proof of Lemma \ref{lem:addition_module}]
The idea is to build the product $\prod_{j=1}^d (x^{(j)})^{n_j}$ in an iterative way: The first step starts with $(x^{(1)})^{n_1}$, then in the next steps the 
factors $(x^{(j)})^{n_j}, j=2,\ldots,d$ will be combined multiplicatively. Finally, the product is added to the sum. We assume $n_j > 0$ for $j=1, \ldots, d$, 
otherwise the inputs with $n_j = 0$ will simply be ignored.
\begin{enumerate}
\item The first step is realized with help of the univariate monomial module (and a linear layer contribution in case of $d>1$) to build 
\begin{align*}
r_1 := \begin{cases}
(x^{(1)})^{n_1} 					& d = 1 \\
(x^{(1)})^{n_1} + \beta x^{(2)} 	& d > 1.
\end{cases}
\end{align*}
This is depicted in Figure \ref{fig:sum_module_construction} left. If $d=1$, the next step is skipped. This step requires a depth of $L_1 = \max(\lceil \log_2(n_1) \rceil, 1)$, see Lemma \ref{lem:univariate_monomial_module}.
\item In the $\ell$-th step, $\ell \in \{2, \ldots, d-1\}$: 
We have $r_{\ell-1} = \prod_{j=1}^{\ell-1} (x^{(j)})^{n_j} + \beta x^{(\ell)}$.
A bivariate monomial module multiplies the intermediate result $r_{\ell-1}$ from the last step with $(x^{(\ell)})^{n_\ell}$. 
Furthermore, a univariate monomial module builds $(x^{(\ell)})^{n_\ell+1}$. 
A subsequent linear layer combines these parts and the possible summand $\beta x^{(\ell+1)}$ to yield the intermediate result $r_\ell$:
\begin{align*}
r_\ell &= r_{\ell-1} \cdot (x^{(\ell)})^{n_\ell} - \beta (x^{(\ell)})^{n_\ell+1} + \beta x^{(\ell+1)}
= \prod_{j=1}^{\ell} (x^{(j)})^{n_j} + \beta x^{(\ell+1)}
\end{align*}
The corresponding setup is depicted in Figure \ref{fig:sum_module_construction}, middle. 
The required depth $L_\ell$ for this step is given by the depth of approximating $(x^{(\ell)})^{n_\ell+1}$, which is given by $L_\ell = \lceil \log_2(n_\ell+1) \rceil$ due to \Cref{lem:univariate_monomial_module} and $n_\ell > 0$.
\item The $d$-th step,
where we have $r_{d-1} = \prod_{j=1}^{d-1} (x^{(j)})^{n_j} + \beta x^{(d)}$,
simply multiplicatively combines the last contribution $(x^{(d)})^{n_d}$. 
Instead of an additional part like $\beta x^{(d+1)}$, 
we will add $\beta x^{(d)}$: 
A subsequent linear layer adds the product $r_d = \prod_{j=1}^d (x^{(j)})^{n_j}$ and the additional "$+ \beta x^{(d)}$" to the sum $S$:
\begin{align*}
r_d &= r_{d-1} \cdot (x^{(d)})^{n_d} - \beta (x^{(d)})^{n_d+1} + \beta x^{(d)}
= \prod_{j=1}^{d} (x^{(j)})^{n_j} + \beta x^{(d)}
\end{align*}
As explained in Remark \ref{rem:additional_beta_term}, the "$+\beta x^{(d)}$ part is required to ensure that the three centers within the last dimension $d+1$ are pairwise distinct, i.e.\ $z_i^{(d+1)} + \prod_{i=1}^d (z_i^{(j)})^{n_j} + \beta z_i^{(d)}$ for $i=1,2,3$ are pairwise distinct. This will be required in order to stack such sum-modules on top of each other later. \\
The required depth for this step can be infered from the bivariate monomial module (Lemma \ref{lem:bivariate_monomial_module}) and is thus given by $L_d = \max(\lceil \log_2(n_d) \rceil, 1)$.
\end{enumerate}
Overall, the depth is limited by 
\begin{align*}
L = \sum_{j=1}^d L_j \leq \sum_{j=1}^d \lceil \log_2(n_j + 1) \rceil \leq d \cdot \lceil \log_2(\max_{j=1,\ldots,d} n_j + 1) \rceil.
\end{align*}
\end{proof}

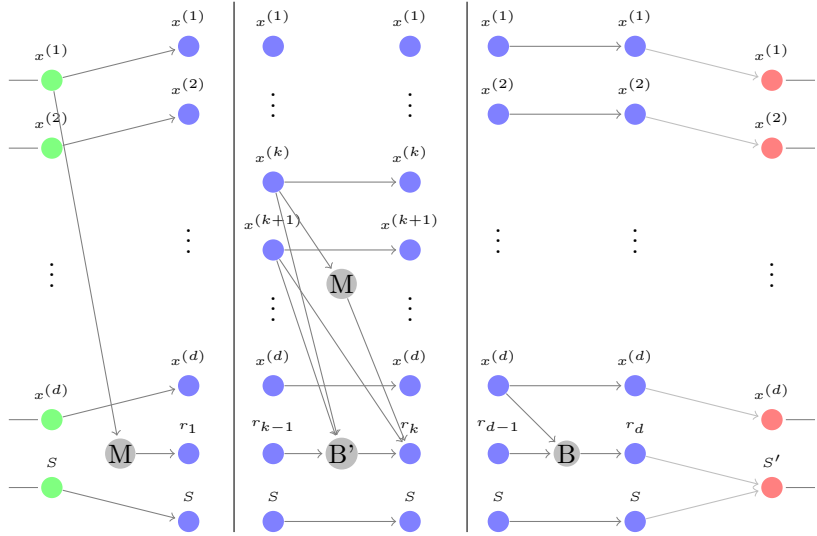
\begin{figure}[H]
\def\layersep{2cm}
\centering %
\begin{tikzpicture}[scale=0.45, shorten >=1pt,->,draw=black!50, node distance=\layersep]
    \tikzstyle{every pin edge}=[<-,shorten <=1pt]
    \tikzstyle{neuron}=[circle,fill=black!25,minimum size=8pt,inner sep=0pt]
    \tikzstyle{input neuron}=[neuron, fill=green!50];
    \tikzstyle{output neuron}=[neuron, fill=red!50];
    \tikzstyle{hidden neuron}=[neuron, fill=blue!50];
    \tikzstyle{relu neuron}=[neuron, fill=black, minimum size=5pt];    
    \tikzstyle{annot} = [text width=4em, text centered]
    \tikzstyle{module}=[circle,fill=black!25,minimum size=8pt,inner sep=0pt]

   	\node[input neuron, pin={[pin edge={-}]left:}, label={\tiny{$x^{(1)}$}}] (I-1) at (0,5) {};
	\node[input neuron, pin={[pin edge={-}]left:}, label={\tiny{$x^{(2)}$}}] (I-2) at (0,3) {};
   	\node[input neuron, pin={[pin edge={-}]left:}, label={\tiny{$x^{(d)}$}}] (I-3) at (0,-5) {};
	\node[input neuron, pin={[pin edge={-}]left:}, label={\tiny{$S$}}] (I-4) at (0,-7) {};	

    \node(dots) at (0,-.5){\vdots};
	\node(dots) at (2*\layersep,.5){\vdots};

	\node[hidden neuron, label={\tiny{$x^{(1)}$}}] (H1-1) at (2*\layersep,6) {};
	\node[hidden neuron, label={\tiny{$x^{(2)}$}}] (H1-2) at (2*\layersep,4) {};
	\node[hidden neuron, label={\tiny{$x^{(d)}$}}] (H1-3) at (2*\layersep,-4) {};		
	\node[hidden neuron, label={\tiny{$r_1$}}] (H1-4) at (2*\layersep,-6) {};		
	\node[hidden neuron, label={\tiny{$S$}}] (H1-5) at (2*\layersep,-8) {};		
	
	\node[module] (B1) at (1*\layersep, -6) {M};    

	\path [gray] (I-1) edge (H1-1);
	\path [gray] (I-2) edge (H1-2);
	\path [gray] (I-3) edge (H1-3);
	\path [gray] (I-4) edge (H1-5);
	
	\path [gray] (I-1) edge (B1);	
	\path [gray] (B1) edge (H1-4);	

\end{tikzpicture}
\vrule
\begin{tikzpicture}[scale=0.45, shorten >=1pt,->,draw=black!50, node distance=\layersep]
    \tikzstyle{every pin edge}=[<-,shorten <=1pt]
    \tikzstyle{neuron}=[circle,fill=black!25,minimum size=8pt,inner sep=0pt]
    \tikzstyle{input neuron}=[neuron, fill=green!50];
    \tikzstyle{output neuron}=[neuron, fill=red!50];
    \tikzstyle{hidden neuron}=[neuron, fill=blue!50];
    \tikzstyle{relu neuron}=[neuron, fill=black, minimum size=5pt];    
    \tikzstyle{annot} = [text width=4em, text centered]
    \tikzstyle{module}=[circle,fill=black!25,minimum size=8pt,inner sep=0pt]

	\node(dots) at (0,4.5){\vdots};
	\node(dots) at (2*\layersep,4.5){\vdots};	
	
	\node(dots) at (0,-1.5){\vdots};
	\node(dots) at (2*\layersep,-1.5){\vdots};

   	\node[hidden neuron, label={\tiny{$x^{(1)}$}}] (H1-1) at (0,6) {};
	\node[hidden neuron, label={\tiny{$x^{(k)}$}}] (H1-2) at (0,2) {};
	\node[hidden neuron, label={\tiny{$x^{(k+1)}$}}] (H1-3) at (0,0) {};
	\node[hidden neuron, label={\tiny{$x^{(d)}$}}] (H1-4) at (0,-4) {};		
	\node[hidden neuron, label={\tiny{$r_{k-1}$}}] (H1-5) at (0,-6) {};		
	\node[hidden neuron, label={\tiny{$S$}}] (H1-6) at (0,-8) {};		

	\node[module] (M1) at (1*\layersep, -1) {M};    
	\node[module] (B1) at (1*\layersep, -6) {B'};    

	\node[hidden neuron, label={\tiny{$x^{(1)}$}}] (H2-1) at (2*\layersep,6) {};	
	\node[hidden neuron, label={\tiny{$x^{(k)}$}}] (H2-2) at (2*\layersep,2) {};
	\node[hidden neuron, label={\tiny{$x^{(k+1)}$}}] (H2-3) at (2*\layersep,0) {};
	\node[hidden neuron, label={\tiny{$x^{(d)}$}}] (H2-4) at (2*\layersep,-4) {};		
	\node[hidden neuron, label={\tiny{$r_k$}}] (H2-5) at (2*\layersep,-6) {};		
	\node[hidden neuron, label={\tiny{$S$}}] (H2-6) at (2*\layersep,-8) {};		

	\path [gray] (H1-2) edge (H2-2);
	\path [gray] (H1-3) edge (H2-3);
	\path [gray] (H1-4) edge (H2-4);
	\path [gray] (H1-6) edge (H2-6);
	\path [gray] (H1-3) edge (H2-5);

	\path [gray] (H1-2) edge (B1);
	\path [gray] (H1-3) edge (B1);
	\path [gray] (H1-5) edge (B1);		
	\path [gray] (H1-2) edge (M1);

	\path [gray] (B1) edge (H2-5);
	\path [gray] (M1) edge (H2-5);

\end{tikzpicture}
\vrule
\begin{tikzpicture}[scale=0.45, shorten >=1pt,->,draw=black!50, node distance=\layersep]
    \tikzstyle{every pin edge}=[<-,shorten <=1pt]
    \tikzstyle{neuron}=[circle,fill=black!25,minimum size=8pt,inner sep=0pt]
    \tikzstyle{input neuron}=[neuron, fill=green!50];
    \tikzstyle{output neuron}=[neuron, fill=red!50];
    \tikzstyle{hidden neuron}=[neuron, fill=blue!50];
    \tikzstyle{relu neuron}=[neuron, fill=black, minimum size=5pt];    
    \tikzstyle{annot} = [text width=4em, text centered]
    \tikzstyle{module}=[circle,fill=black!25,minimum size=8pt,inner sep=0pt]

	\node(dots) at (0,.5){\vdots};
	\node(dots) at (2*\layersep,.5){\vdots};
	\node(dots) at (4*\layersep,-.5){\vdots};

	\node[hidden neuron, label={\tiny{$x^{(1)}$}}] (H1-1) at (0,6) {};
	\node[hidden neuron, label={\tiny{$x^{(2)}$}}] (H1-2) at (0,4) {};
	\node[hidden neuron, label={\tiny{$x^{(d)}$}}] (H1-3) at (0,-4) {};		
	\node[hidden neuron, label={\tiny{$r_{d-1}$}}] (H1-4) at (0,-6) {};		
	\node[hidden neuron, label={\tiny{$S$}}] (H1-5) at (0,-8) {};		

	\node[module] (B1) at (1*\layersep, -6) {B};

	\node[hidden neuron, label={\tiny{$x^{(1)}$}}] (H2-1) at (2*\layersep,6) {};
	\node[hidden neuron, label={\tiny{$x^{(2)}$}}] (H2-2) at (2*\layersep,4) {};	
	\node[hidden neuron, label={\tiny{$x^{(d)}$}}] (H2-3) at (2*\layersep,-4) {};		
	\node[hidden neuron, label={\tiny{$r_d$}}] (H2-4) at (2*\layersep,-6) {};		
	\node[hidden neuron, label={\tiny{$S$}}] (H2-5) at (2*\layersep,-8) {};	

	\node[output neuron, pin={[pin edge={-}]right:}, label={\tiny{$x^{(1)}$}}] (O-1) at (4*\layersep,5) {};
	\node[output neuron, pin={[pin edge={-}]right:}, label={\tiny{$x^{(2)}$}}] (O-2) at (4*\layersep,3) {};	
	\node[output neuron, pin={[pin edge={-}]right:}, label={\tiny{$x^{(d)}$}}] (O-3) at (4*\layersep,-5) {};		
    \node[output neuron, pin={[pin edge={-}]right:}, label={\tiny{$S'$}}] (O-4) at (4*\layersep, -7) {};	

   	\path [gray] (H1-1) edge (H2-1);
	\path [gray] (H1-2) edge (H2-2);
	\path [gray] (H1-3) edge (H2-3);
	\path [gray] (H1-5) edge (H2-5);

	\path [gray] (H1-3) edge (B1);
	\path [gray] (H1-4) edge (B1);		
	
	\path [gray] (B1) edge (H2-4);

	\path [lightgray] (H2-1) edge (O-1);
	\path [lightgray] (H2-2) edge (O-2);
	\path [lightgray] (H2-3) edge (O-3);
	\path [lightgray] (H2-4) edge (O-4);
	\path [lightgray] (H2-5) edge (O-4);

\end{tikzpicture}
\caption{Visualization of the construction in Lemma \ref{lem:addition_module}. Left: Adding a monomial in $x^{(1)}$. 
Middle: For $\ell \in \{2, \ldots, d-1\}$, 
extend $r_{\ell-1} = \prod_{j=1}^{\ell-1} (x^{(j)})^{n_j} + \beta x^{(\ell)}$ to $r_{\ell} = \prod_{j=1}^{\ell} (x^{(j)})^{n_i} + \beta x^{(\ell+1)}$. 
Right: Extend $r_{d-1} = \prod_{j=1}^{d-1} (x^{(j)})^{n_j} + \beta x^{(d)}$ to $r_{d} = \prod_{j=1}^{d} (x^{(j)})^{n_j}$ and add it to $S$. 
The grey circles ``M'', ``B'' and ``B$^\prime$'' refer to the modules introduced earlier.}
\label{fig:sum_module_construction}
\end{figure}

Now it is time to state and prove the main result of this section, which is based on an iterative application of Lemma \ref{lem:addition_module}.

\begin{theorem}[Universal approximation for unbounded depth] 
\label{th:infinite_depth}
Let $\Omega \subset \R_{\geq 0}^d, d \in \N$ be a compact domain. Consider an arbitrary continuous function $f: \Omega \rightarrow \R$. Then it is possible under the Assumptions \ref{ass:deep_assumptions} to approximate this function $f$ to arbitrary accuracy using an SDKN of width $w = d+8$ and $M=3$ centers, i.e.\
\begin{align*}
  \lim_{L \rightarrow \infty} \dist(f, \mathcal{F}_{L, d+8, 3}(\Omega))) = 0.
\end{align*}
\end{theorem}

\begin{proof}
By stacking the addition modules from Lemma \ref{lem:addition_module} after each other, 
it is possible to approximate any polynomial $\sum_{l \in \N^d} \alpha_l x^l$ to arbitrary accuracy, 
where $l \in \N^d$ is a multiindex. 
This means that $\bigcup_{L \in \N} \mathcal{F}_{L,d+8, 3}$ is dense in the space of polynomials. 
As the space of polynomials is dense in the continuous functions (Stone--Weierstrass) and denseness is a transitive property, 
this implies that $\bigcup_{L \in \N} \mathcal{F}_{L, d+8, 3}(\Omega)$ is dense in $C(\Omega)$.
\end{proof}

\begin{rem} \label{rem:deep_case_construction} $\mbox{}$ \\
\begin{enumerate}
\item In the theory of neural networks, most of the results on deep approximation focus on ReLU activation function, as it allows to realize explicit constructions.
In order to approximate the squaring operation $x \mapsto x^2$, a multilayer \textit{saw-tooth} construction is used \cite{YAROTSKY2017103}. 
This is not required here: 
Due to the increased flexibility and extensive theory of kernels, the squaring operation can be realized without multiple layers, see \Cref{lem:basic_operations}. 
This is a decisive theoretical advantage compared to ReLU neural networks, as it allows for simpler networks, also analytically. \\
Furthermore, the above constructions hold for a whole class of kernels (as specified via \Cref{ass:deep_assumptions}) instead of only for a single activation function.
\item Nevertheless, in~\cite{math7100992}, Theorem 1, a result analogous to \Cref{th:infinite_depth} is proven for ReLU neural networks, obtaining the same 
type of universality with depth $w=d+2$, where additionally an estimate of the depth $L$ needed to reach a certain target accuracy in the $L_\infty$ norm is 
given. 
Obtaining similar estimates with our method will be the object of future work. In general, we did not derive statements on %
  expression rates, as this was not the intention here. 
\item The Wendland kernel of order $0$ is given by $k(x,y) = \max(1 - |x-y|, 0)$.
Due to its piecewise affine shape, it is quite comparable to the ReLU function. 
But the Wendland kernel of order $0$ does not satisfy the \Cref{ass:deep_assumptions}. 
However it turns out that most of the constructions proven in \Cref{subsubsec:infinite_depth_case} also work for the Wendland kernel of order $0$ by using its similarity to the ReLU-function and the sawtooth construction for ReLU neural networks \cite{YAROTSKY2017103}.
\end{enumerate}
\end{rem}

\section{Conclusion and Outlook}
In this work we examined how to use a deep kernel representer theorem to build powerful deep kernel models, 
and in particular we introduced a number of apt architectural choices that resulted in the definition of the Structured Deep Kernel Network (SDKN) setup. 
The SDKN can be interpreted as a neural network with optimizable activation functions, governed by a (deep kernel) representer theorem. 
This formulation offers increased approximation power compared to standard neural networks. 
The analysis of SDKNs leverages tools from both kernel theory and neural network research. 
Through a rigorous theoretical investigation, we established fundamental properties of this novel class of kernel models, including several universal approximation results. 
Notably, in the deep case, the SDKN allows for asymptotically superior constructions compared to ReLU-based neural networks

Further directions of research include the derivation of quantitative instead of only qualitative approximation statements, namely, addressing the problem of the speed of approximation for certain function classes.
Such results exist for neural networks, e.g.\ \cite{schwabzech2023,L21}.
Hence investigation of such quantitative approximation rate statements is a
natural next step.
Also the closer investigation of regularization in terms of RKHS norms is part
of future work. 
Another key target is to make use of the resulting deep kernel from Equation \eqref{eq:deep_kernel}
in order to enhance standard kernel methods, e.g.\ greedy kernel methods. 
\subsection*{Acknowledgments}
The authors acknowledge the funding of the project by the Deutsche Forschungsgemeinschaft (DFG, German Research Foundation) 
under
Germany's Excellence Strategy - EXC 2075 - 390740016 and funding by the BMBF under contract 05M20VSA.

\bibliography{references} %

\begin{thebibliography}{10}

\bibitem{NEURIPS2019_62dad6e2}
Z.~Allen-Zhu, Y.~Li, and Y.~Liang.
\newblock Learning and generalization in overparameterized neural networks,
  going beyond two layers.
\newblock In H.~Wallach, H.~Larochelle, A.~Beygelzimer, F.~d'Alch\'{e} Buc,
  E.~Fox, and R.~Garnett, editors, {\em Advances in Neural Information
  Processing Systems}, volume~32. Curran Associates, Inc., 2019.

\bibitem{Bohn2017}
B.~Bohn, C.~Rieger, and M.~Griebel.
\newblock A representer theorem for deep kernel learning.
\newblock {\em Journal of Machine Learning Research}, 20:1--32, 2019.

\bibitem{BuiTan2024}
T.~Bui-Thanh.
\newblock A unified and constructive framework for the universality of neural
  networks.
\newblock {\em IMA Journal of Applied Mathematics}, 89(1):197--230, 2024.

\bibitem{NIPS2009_3628}
Y.~Cho and L.~K. Saul.
\newblock Kernel methods for deep learning.
\newblock In Y.~Bengio, D.~Schuurmans, J.~D. Lafferty, C.~K.~I. Williams, and
  A.~Culotta, editors, {\em Advances in Neural Information Processing Systems
  22}, pages 342--350. Curran Associates, Inc., 2009.

\bibitem{damianou2013deep}
A.~Damianou and N.~D. Lawrence.
\newblock Deep {G}aussian processes.
\newblock In C.~M. Carvalho and P.~Ravikumar, editors, {\em Proceedings of the
  Sixteenth International Conference on Artificial Intelligence and
  Statistics}, volume~31 of {\em Proceedings of Machine Learning Research},
  pages 207--215, Scottsdale, Arizona, USA, 2013.

\bibitem{diederichs2019improved}
B.~Diederichs and A.~Iske.
\newblock Improved estimates for condition numbers of radial basis function
  interpolation matrices.
\newblock {\em Journal of Approximation Theory}, 238:38--51, 2019.

\bibitem{DRISCOLL2002413}
T.~Driscoll and B.~Fornberg.
\newblock Interpolation in the limit of increasingly flat radial basis
  functions.
\newblock {\em Computers \& Mathematics with Applications}, 43(3):413--422,
  2002.

\bibitem{elbrachter2021deep}
D.~Elbr{\"a}chter, D.~Perekrestenko, P.~Grohs, and H.~B{\"o}lcskei.
\newblock Deep neural network approximation theory.
\newblock {\em IEEE Transactions on Information Theory}, 67(5):2581--2623,
  2021.

\bibitem{pmlr-v49-eldan16}
R.~Eldan and O.~Shamir.
\newblock The power of depth for feedforward neural networks.
\newblock In V.~Feldman, A.~Rakhlin, and O.~Shamir, editors, {\em 29th Annual
  Conference on Learning Theory}, volume~49 of {\em Proceedings of Machine
  Learning Research}, pages 907--940. PMLR, 2016.

\bibitem{Fasshauer2015}
G.~E. Fasshauer and M.~McCourt.
\newblock {\em Kernel-{B}ased {A}pproximation {M}ethods Using {MATLAB}}.
\newblock World Scientific Publishing Co. Pte. Ltd., Hackensack, {NJ}, 2015.

\bibitem{girosi1989representation}
F.~Girosi and T.~Poggio.
\newblock Representation properties of networks: Kolmogorov's theorem is
  irrelevant.
\newblock {\em Neural Computation}, 1(4):465--469, 1989.

\bibitem{Goodfellow2016}
I.~Goodfellow, Y.~Bengio, and A.~Courville.
\newblock {\em Deep Learning}.
\newblock MIT Press, 2016.
\newblock \url{http://www.deeplearningbook.org}.

\bibitem{guhring2020error}
I.~G{\"u}hring, G.~Kutyniok, and P.~Petersen.
\newblock Error bounds for approximations with deep {ReLU} neural networks in
  ${W}^{s, p}$ norms.
\newblock {\em Analysis and Applications}, 18(05):803--859, 2020.

\bibitem{goehring2021bsc}
M.~Göhring.
\newblock Model order redution with kernel autoencoders.
\newblock Bachelor thesis, University of Stuttgart, 2021.

\bibitem{hamzi2021simple}
B.~Hamzi, R.~Maulik, and H.~Owhadi.
\newblock Simple, low-cost and accurate data-driven geophysical forecasting
  with learned kernels.
\newblock {\em Proceedings of the Royal Society A}, 477(2252):20210326, 2021.

\bibitem{hamzi2021learning}
B.~Hamzi and H.~Owhadi.
\newblock Learning dynamical systems from data: a simple cross-validation
  perspective, part i: parametric kernel flows.
\newblock {\em Physica D: Nonlinear Phenomena}, 421:132817, 2021.

\bibitem{hamzi2023learning}
B.~Hamzi, H.~Owhadi, and Y.~Kevrekidis.
\newblock Learning dynamical systems from data: A simple cross-validation
  perspective, part iv: case with partial observations.
\newblock {\em Physica D: Nonlinear Phenomena}, 454:133853, 2023.

\bibitem{math7100992}
B.~Hanin.
\newblock Universal function approximation by deep neural nets with bounded
  width and {ReLU} activations.
\newblock {\em Mathematics}, 7(10), 2019.

\bibitem{HORNIK1989359}
K.~Hornik, M.~Stinchcombe, and H.~White.
\newblock Multilayer feedforward networks are universal approximators.
\newblock {\em Neural Networks}, 2(5):359--366, 1989.

\bibitem{jacot2018neural}
A.~Jacot, F.~Gabriel, and C.~Hongler.
\newblock Neural tangent kernel: Convergence and generalization in neural
  networks.
\newblock {\em {A}dvances in {N}eural {I}nformation {P}rocessing {S}ystems},
  31, 2018.

\bibitem{kidger2020universal}
P.~Kidger and T.~Lyons.
\newblock Universal approximation with deep narrow networks.
\newblock In {\em Conference on learning theory}, pages 2306--2327. PMLR, 2020.

\bibitem{Wahba1970}
G.~S. Kimeldorf and G.~Wahba.
\newblock A correspondence between {B}ayesian estimation on stochastic
  processes and smoothing by splines.
\newblock {\em The Annals of Mathematical Statistics}, 41(2):495--502, 1970.

\bibitem{L21}
S.~Langer.
\newblock Approximating smooth functions by deep neural networks with sigmoid
  activation function.
\newblock {\em Journal of Multivariate Analysis}, 182:104696, 2021.

\bibitem{LARSSON2005103}
E.~Larsson and B.~Fornberg.
\newblock Theoretical and computational aspects of multivariate interpolation
  with increasingly flat radial basis functions.
\newblock {\em Computers \& Mathematics with Applications}, 49(1):103--130,
  2005.

\bibitem{Le2016}
L.~Le, J.~Hao, Y.~Xie, and J.~Priestley.
\newblock Deep kernel: Learning kernel function from data using deep neural
  network.
\newblock In {\em 2016 IEEE/ACM 3rd International Conference on Big Data
  Computing Applications and Technologies (BDCAT)}, pages 1--7, 2016.

\bibitem{lee2023learning}
J.~Lee, E.~De~Brouwer, B.~Hamzi, and H.~Owhadi.
\newblock Learning dynamical systems from data: A simple cross-validation
  perspective, part iii: Irregularly-sampled time series.
\newblock {\em Physica D: Nonlinear Phenomena}, 443:133546, 2023.

\bibitem{LeshnoLPS93}
M.~Leshno, V.~Y. Lin, A.~Pinkus, and S.~Schocken.
\newblock Multilayer feedforward networks with a nonpolynomial activation
  function can approximate any function.
\newblock {\em Neural Networks}, 6(6):861--867, 1993.

\bibitem{liu2024kan}
Z.~Liu, Y.~Wang, S.~Vaidya, F.~Ruehle, J.~Halverson, M.~Solja{\v{c}}i{\'c},
  T.~Y. Hou, and M.~Tegmark.
\newblock {KAN}: {K}olmogorov-{A}rnold {N}etworks.
\newblock {\em arXiv preprint arXiv:2404.19756}, 2024.

\bibitem{lorentz1996constructive}
G.~Lorentz, M.~von Golitschek, and Y.~Makovoz.
\newblock {\em Constructive approximation: advanced problems}.
\newblock Grundlehren der mathematischen Wissenschaften. Springer, 1996.

\bibitem{Ma2017}
S.~Ma and M.~Belkin.
\newblock Diving into the shallows: A computational perspective on large-scale
  shallow learning.
\newblock In {\em Proceedings of the 31st International Conference on Neural
  Information Processing Systems}, NIPS'17, page 3781–3790. Curran Associates
  Inc., 2017.

\bibitem{Ma2019b}
S.~Ma and M.~Belkin.
\newblock Kernel machines that adapt to {GPUs} for effective large batch
  training.
\newblock In {\em Conference on Systems and Machine Learning (SysML)}, 2019.

\bibitem{Meanti2020}
G.~Meanti, L.~Carratino, L.~Rosasco, and A.~Rudi.
\newblock Kernel methods through the roof: Handling billions of points
  efficiently.
\newblock In H.~Larochelle, M.~Ranzato, R.~Hadsell, M.~F. Balcan, and H.~Lin,
  editors, {\em Advances in Neural Information Processing Systems}, volume~33,
  pages 14410--14422. Curran Associates, Inc., 2020.

\bibitem{Micchelli2006}
C.~A. Micchelli, Y.~Xu, and H.~Zhang.
\newblock Universal kernels.
\newblock {\em Journal of Machine Learning Research}, 7:2651--2667, 2006.

\bibitem{montanelli2020error}
H.~Montanelli and H.~Yang.
\newblock Error bounds for deep {ReLU} networks using the
  {K}olmogorov--{A}rnold superposition theorem.
\newblock {\em Neural Networks}, 129:1--6, 2020.

\bibitem{M21}
S.~Moon.
\newblock {ReLU} network with bounded width is a universal approximator in view
  of an approximate identity.
\newblock {\em Applied Sciences}, 11(1), 2021.

\bibitem{MULLER2009645}
S.~Müller and R.~Schaback.
\newblock A {N}ewton basis for kernel spaces.
\newblock {\em Journal of Approximation Theory}, 161(2):645--655, 2009.

\bibitem{OK19}
I.~Ohn and Y.~Kim.
\newblock Smooth function approximation by deep neural networks with general
  activation functions.
\newblock {\em Entropy}, 21(7):627, 2019.

\bibitem{OWHADI201922}
H.~Owhadi and G.~R. Yoo.
\newblock Kernel flows: From learning kernels from data into the abyss.
\newblock {\em Journal of Computational Physics}, 389:22--47, 2019.

\bibitem{NEURIPS2019_9015}
A.~Paszke, S.~Gross, F.~Massa, A.~Lerer, J.~Bradbury, G.~Chanan, T.~Killeen,
  Z.~Lin, N.~Gimelshein, L.~Antiga, A.~Desmaison, A.~Kopf, E.~Yang, Z.~DeVito,
  M.~Raison, A.~Tejani, S.~Chilamkurthy, B.~Steiner, L.~Fang, J.~Bai, and
  S.~Chintala.
\newblock Pytorch: An imperative style, high-performance deep learning library.
\newblock In {\em Advances in Neural Information Processing Systems 32}, pages
  8024--8035. Curran Associates, Inc., 2019.

\bibitem{pinkus1999approximation}
A.~Pinkus.
\newblock Approximation theory of the {MLP} model in neural networks.
\newblock {\em Acta Numerica}, 8(1):143--195, 1999.

\bibitem{Rahimi2008}
A.~Rahimi and B.~Recht.
\newblock Random features for large-scale kernel machines.
\newblock In {\em Advances in Neural Information Processing Systems 20}. MIT
  Press, 2008.

\bibitem{SCARDAPANE201919}
S.~Scardapane, S.~{Van Vaerenbergh}, S.~Totaro, and A.~Uncini.
\newblock Kafnets: Kernel-based non-parametric activation functions for neural
  networks.
\newblock {\em Neural Networks}, 110:19--32, 2019.

\bibitem{Schoelkopf2001v}
B.~Sch{\"o}lkopf, R.~Herbrich, and A.~J. Smola.
\newblock A generalized representer theorem.
\newblock In D.~Helmbold and B.~Williamson, editors, {\em Computational
  Learning Theory}, pages 416--426. Springer Berlin Heidelberg, 2001.

\bibitem{schwabzech2023}
C.~Schwab and J.~Zech.
\newblock Deep learning in high dimension: Neural network expression rates for
  analytic functions in ${L}^2(\mathbb{R}^d, \gamma_d)$.
\newblock {\em SIAM/ASA Journal on Uncertainty Quantification}, 11(1):199--234,
  2023.

\bibitem{steinwart2016learning}
I.~Steinwart, P.~Thomann, and N.~Schmid.
\newblock Learning with hierarchical {G}aussian kernels.
\newblock {\em arXiv preprint arXiv:1612.00824}, 2016.

\bibitem{pmlr-v49-telgarsky16}
M.~Telgarsky.
\newblock Benefits of depth in neural networks.
\newblock In V.~Feldman, A.~Rakhlin, and O.~Shamir, editors, {\em 29th Annual
  Conference on Learning Theory}, volume~49 of {\em Proceedings of Machine
  Learning Research}, pages 1517--1539. PMLR, 23--26 Jun 2016.

\bibitem{JMLR:v20:18-418}
M.~Unser.
\newblock A representer theorem for deep neural networks.
\newblock {\em Journal of Machine Learning Research}, 20(110):1--30, 2019.

\bibitem{wang2022bsc}
J.~Wang.
\newblock Solving initial value problems with neural network models and kernel
  network models.
\newblock Bachelor thesis, University of Stuttgart, 2022.

\bibitem{Wendland2005}
H.~Wendland.
\newblock {\em Scattered {D}ata {A}pproximation}, volume~17.
\newblock Cambridge University Press, 2005.

\bibitem{wenzel2024application}
T.~Wenzel, B.~Haasdonk, H.~Kleikamp, M.~Ohlberger, and F.~Schindler.
\newblock {A}pplication of {D}eep {K}ernel {M}odels for {C}ertified and
  {A}daptive {RB}-{ML}-{ROM} {S}urrogate {M}odeling.
\newblock In I.~Lirkov and S.~Margenov, editors, {\em Large-Scale Scientific
  Computations}, pages 117--125, Cham, 2024. Springer Nature Switzerland.

\bibitem{wenzel2022structured}
T.~Wenzel, M.~Kurz, A.~Beck, G.~Santin, and B.~Haasdonk.
\newblock Structured {D}eep {K}ernel {N}etworks for {D}ata-{D}riven {C}losure
  {T}erms of {T}urbulent {F}lows.
\newblock In {\em Large-Scale Scientific Computing}, pages 410--418. Springer
  International Publishing, 2022.

\bibitem{Williams2001}
C.~Williams and M.~Seeger.
\newblock Using the {Nystr\"{o}m} method to speed up kernel machines.
\newblock In T.~Leen, T.~Dietterich, and V.~Tresp, editors, {\em Advances in
  Neural Information Processing Systems}, volume~13, pages 682--688. MIT Press,
  2001.

\bibitem{Wilson2016}
A.~G. Wilson, Z.~Hu, R.~Salakhutdinov, and E.~P. Xing.
\newblock Deep kernel learning.
\newblock In A.~Gretton and C.~C. Robert, editors, {\em Proceedings of the 19th
  International Conference on Artificial Intelligence and Statistics},
  volume~51, pages 370--378, 2016.

\bibitem{wurzberger2024learning}
F.~Wurzberger and F.~Schwenker.
\newblock Learning in deep radial basis function networks.
\newblock {\em Entropy}, 26(5):368, 2024.

\bibitem{ijcai2019-558}
H.~Xue, Z.-F. Wu, and W.-X. Sun.
\newblock Deep spectral kernel learning.
\newblock In {\em Proceedings of the Twenty-Eighth International Joint
  Conference on Artificial Intelligence, {IJCAI-19}}, pages 4019--4025.
  International Joint Conferences on Artificial Intelligence Organization,
  2019.

\bibitem{yang2021tensor}
G.~Yang and E.~J. Hu.
\newblock Tensor programs iv: Feature learning in infinite-width neural
  networks.
\newblock In {\em International Conference on Machine Learning}, pages
  11727--11737. PMLR, 2021.

\bibitem{YAROTSKY2017103}
D.~Yarotsky.
\newblock Error bounds for approximations with deep {ReLU} networks.
\newblock {\em Neural Networks}, 94:103--114, 2017.

\end{thebibliography}
\bibliographystyle{abbrv}

\appendix

\section{Further proofs} \label{sec:further_proofs}

\subsection{Positive definiteness of the employed kernels}\label{app:positive_definiteness}

Both classes of kernels introduced in Section \ref{subsec:setup_sdkn} are positive definite kernels, but in general not strictly positive definite:

\begin{enumerate}
\item The linear kernel introduced in Section \ref{subsubsec:linear_kernel} is positive definite but in general not strictly positive definite,
  as its feature map is the identity mapping, i.e.\ the RKHS is especially finite-dimensional. 
\item The single-dimensional matrix-valued kernel $k_s$ is also in general only positive definite, even if the used kernels $k^{(1)}, .., k^{(d)}$ are strictly 
ositive definite. This can be observed if we use $k^{(1)} = .. = k^{(d)}$: The overall kernel matrix related to single-dimensional matrix-valued kernels is the 
diagonal concatenation of the kernel matrices concerning the single dimensions. Zero eigenvalues can occur, if some coordinate of different points coincide: 
Consider two data points $x_1 \neq x_2 \in \R^2$ with $x_1^{(1)} = x_2^{(1)}$, i.e.\ the first coordinate coincides. When the Wendland kernel of order $0$ is 
used as the single-dimensional kernel, the resulting kernel matrix is rank deficient, e.g.

\begin{align*}
X_2 = \left\{ 
\begin{pmatrix}
    1 \\ -0.2 
\end{pmatrix}
,
\begin{pmatrix}    
    1 \\ -0.9
\end{pmatrix}
\right\} 
\Rightarrow
(k_s(x_i, x_j))_{i,j} =
\begin{pmatrix}
1 &		1 & 	0 & 		0 \\
1 & 	1 & 	0 &     	0 \\
0 &    	0 &		1 & 		0.3 \\
0 &		0 &		0.3 & 	1 \\
\end{pmatrix}
\end{align*}
with eigenvalues  $0$, $0.7$, $1.3$, $2$.
\end{enumerate}

\subsection{Proofs}
\begin{delayedproof}{prop:linear_layer}
We consider $b=1$. The extension to $b>1$ is obvious. Thus $A \in \R^{b \times d}$ is simply a row vector and the linear mapping $x \mapsto Ax$ simplifies to a dot product, i.e.\ $Ax = \langle A^\top, x \rangle_{\R^d}$. \\
Using the linear kernel, the kernel map $s$ reads 
\begin{align*}
s(x) &= \sum_{i=1}^M \alpha_i k(x, z_i) = \sum_{i=1}^M \alpha_i \langle x, z_i \rangle_{\R^d} = \langle x, \sum_{i=1}^M \alpha_i z_i \rangle_{\R^d}
\end{align*}
Thus we have $Ax = s(x)$ iff there are $\alpha_i \in \R$ such that $\sum_{i=1}^M \alpha_i z_i = A^\top$. 
\end{delayedproof}

\begin{delayedproof}{prop:linear_layer_2}
We consider $b=1$. The extension to $b>1$ is obvious. Thus assume $\beta_j \in \R$ for $j=1,..,d$. Set $\beta = (\beta_1, .., \beta_d)^T \in \R^d$ and decompose $\beta = \beta^\parallel + \beta^\perp$ with $\beta^\parallel \in \Sp g(\Omega) \subset \R^d, \beta^\perp \perp \Sp g(\Omega)$. 
\begin{enumerate}
\item The mapping 
\begin{align*}
(g^{(1)}(x), .., g^{(d)}(x))^\top \mapsto \sum_{j=1}^d \beta_j^\parallel g^{(j)}(x) = \langle \beta^\parallel, g(x) \rangle_{\R^d}
\end{align*}
can be realized as a kernel mapping, if the centers $z_1, .., z_d \in \Omega$ are chosen such that $\Sp \{ g(z_i), i=1,..,d\} \supset \Sp g(\Omega)$: Like this, we have $\beta^\parallel \in \Sp g(\Omega) = \Sp \{ g(z_1), .., g(z_d) \}$ and \Cref{prop:linear_layer} yields the representation via the linear kernel. 
\item For $\beta^\perp$ we have
\begin{align*}
\sum_{j=1}^d \beta_j^\perp g^{(j)}(x) = \langle \beta^\perp, g(x) \rangle_{\R^d} = 0 
\end{align*}
as $g(x) \in \Sp g(\Omega), \beta^\perp \perp \Sp g(\Omega)$. 
\end{enumerate}
In conclusion this yields that $\langle \beta, g(x) \rangle_{\R^d} =
  \langle \beta^\parallel, g(x) \rangle_{\R^d}$ can be expressed as linear kernel mapping.
\end{delayedproof}

\begin{delayedproof}{th:unbounded_nr_centers}
Using the Kolmogorov-Arnold Theorem \ref{th:kolmogorov_theorem}, we can represent the given function $f$ as 
\begin{align*}
f(x^{(1)}, .., x^{(d)}) = \sum_{q=0}^{2d} \Phi \left( \sum_{j=1}^d \lambda_j \phi_{q}(x^{(j)}) \right)
\end{align*}
with continuous $\Phi$, $\phi_{q}$ for $q=0, .., 2d$, $j=1, .., d$. 
For our construction, we make use of the function decomposition depicted in Figure \ref{fig:visualization_NN_DK_inf_centers}.
\begin{enumerate}
\item Approximation of $\phi_{q}: [0, 1] \rightarrow [0, 1]$: As the kernel $k_2$ is universal, it holds that the RKHS $\mathcal{H}_{k_2}([0,1]) \subset C([0,1])$ is a dense subspace in the continuous functions. By construction of the RKHS holds
  $$\overline{\text{span}\{k_2(\cdot, u), u \in [0,1]\}} = \mathcal{H}_{k_2}([0,1]).$$
  Therefore, since all $\phi_{q}, q=0, .., 2d$ are continuous, we conclude
\begin{align*}
\forall \epsilon > 0 ~ &\exists N^{(q)} \in \N, \{u_i^{(q)}\}_{i=1}^{N^{(q)}} \subset [0, 1], \{\alpha_i^{(q)}\}_{i=1}^{N^{(q)}} \subset \R ~ \\ 
&\text{such that} ~ \left\lVert \phi_{q}(\cdot) - \sum_{i=1}^{N^{(q)}} \alpha_i^{(q)} k_2(\cdot, u_i^{(q)}) \right\rVert_{L^\infty([0,1])} < \epsilon.
\end{align*}
We define the abbreviation $\sum_{i=1}^{N^{(q)}} \alpha_i^{(q)} k(\cdot, u_i^{(q)}) =: \tilde{\phi}_{q}(\cdot)$. 
As the argument before holds for any $q=0,.., 2d$, we have for all $x^{(j)} \in [0, 1]$
\begin{align} \label{eq:proof_estimate_1}
\begin{aligned}
\forall q = 0, .., 2d \qquad &\left \lVert \sum_{j=1}^d \lambda_j \phi_{q}(x^{(j)}) - \sum_{j=1}^d \lambda_j \tilde{\phi}_{q}(x^{(j)}) \right \rVert_\infty \\
&\leq \sum_{j=1}^d |\lambda_j| \left \lVert \phi_{q}(x^{(j)}) - \tilde{\phi}_{q}(x^{(j)}) \right \rVert_\infty \leq \epsilon \cdot \sum_{j=1}^d \lambda_j \leq \epsilon.
\end{aligned}
\end{align}
For any $q = 0, .., 2d$, $i=1, .., N^{(q)}$ we define $u_i^{(q, 0)} := u_i^{(q)} \cdot (1, .., 1)^\top \in \R^d$. Using this definition, we have $\mathbb{P}_j u_i^{(q, 0)} = u_i^{(q)}$ for all $j = 1, .., d$, 
here $\mathbb{P}_j: [0, 1]^d \rightarrow [0, 1], x \mapsto x^{(j)}$ denotes the projection operator mapping to the $j$-th coordinate. \\ 
The points $\bigcup_{q=0}^{2d} \{ u_i^{(q, 0)} \}_{i=1}^{N^{(q)}} \subset [0, 1]^d$  are used as centers. \\

\item Approximation of $\Phi$ on $I_q := (\sum_{j=1}^d \lambda_j \tilde{\phi}_{q}((\cdot)^{(j)}))([0,1]^d) \subset \R$ for $q=0, .., 2d$: We employ the same reasoning as in the first step: The RKHS of $k_4$ is dense in the continuous functions, and as $\Phi$ is continuous on $I_q$ it holds
\begin{align} \label{eq:proof_estimate_2.1}
\forall \epsilon > 0 ~ &\exists M^{(q)} \in \N, \{v_i^{(q)}\}_{i=1}^{M^{(q)}} \subset I_q, \{\beta_i^{(q)}\}_{i=1}^{M^{(q)}} \subset \R \notag \\ 
&\text{such that} ~ \left\lVert \sum_{i=1}^{M^{(q)}} \beta_i^{(q)} k_4(\cdot, v_i^{(q)}) - \Phi(\cdot) \right\rVert_{L^\infty(I_q)} < \epsilon.
\end{align}
We define the abbreviation $\sum_{i=1}^{M^{(q)}} \beta_i^{(q)} k_4(\cdot, v_i^{(q)}) =: \tilde{\Phi}_q(\cdot)$. \\
As $I_q$ is the range of $\sum_{j=1}^d \lambda_j \tilde{\phi}_{q}((\cdot)^{(j)})$ on $[0,1]^d$, there exists some $v_i^{(q,0)} \in [0,1]^d$ for every $v_i^{(q)} \in I_q$ such that $\sum_{j=1}^d \lambda_j \tilde{\phi}_{q} \left( \left( v_i^{(q, 0)} \right)^{(j)} \right)  = v_i^{(q)}$.
                                    
\item We use the abbreviation $\phi_{q,j}(\cdot) := \lambda_j \phi_{q}(\cdot)$ and $\tilde{\phi}_{q,j}(\cdot) := \lambda_j \tilde{\phi}_{q}(\cdot)$, $j=1, .., d$. Combining the two previous steps yields: 
\begin{align*}
&\Phi \left( \sum_{j=1}^d \phi_{q,j}(x^{(j)}) \right) - \tilde{\Phi}_q \left( \sum_{j=1}^d \tilde{\phi}_{q,j}(x^{(j)}) \right) \\ 
= &\underbrace{\Phi \left( \sum_{j=1}^d \phi_{q,j}(x^{(j)}) \right) - \Phi \left( \sum_{j=1}^d \tilde{\phi}_{q,j}(x^{(j)}) \right)}_{(*_1)} + \underbrace{\Phi \left( \sum_{j=1}^d \tilde{\phi}_{q,j}(x^{(j)}) \right) - \tilde{\Phi}_q \left( \sum_{j=1}^d \tilde{\phi}_{q,j}(x^{(j)}) \right)}_{(*_2)}
\end{align*}
\begin{enumerate}
\item[$(*_1)$:] The set $(\sum_{j=1}^d \phi_{q,j}((\cdot)^{(j)}))([0,1]^d) \cup (\sum_{j=1}^d \tilde{\phi}_{q,j}((\cdot)^{(j)}))([0,1]^d)$ $\subset \R$ is a compact set. Therefore, $\Phi$ is uniformly continuous on this set. Thus the estimate from Eq.\ \eqref{eq:proof_estimate_1} of the first step allows to bound $(*_1)$.
\item[$(*_2)$:] This difference is bounded by Eq.\ \eqref{eq:proof_estimate_2.1} from the second step of the proof. There, $\Vert \tilde{\Phi}_q(\cdot) - \Phi(\cdot) \Vert_{L^\infty(I_q)} < \epsilon$ was proven.
\end{enumerate}
\item Finally the steps one, two and three can be applied for all $q = 0, .., 2d$. For this, define $\tilde{f}(x) = \sum_{q=0}^{2d} \tilde{\Phi}_q(\sum_{j=1}^d \tilde{\phi}_{q,j}(x^{(j)}))$. Then it holds
\begin{align*}
\Vert f - \tilde{f} \Vert_{L^\infty([0,1]^d)} &= \left \lVert \sum_{j=0}^{2d} \left( \Phi \left( \sum_{j=1}^d \phi_{q,j}(x^{(j)}) \right) - \tilde{\Phi}_q \left( \sum_{j=1}^d \tilde{\phi}_{q,j}(x^{(j)}) \right) \right) \right \rVert_{L^\infty([0,1]^d)} \\
&\leq \sum_{q=0}^{2d} \left \lVert \Phi \left( \sum_{j=1}^d \phi_{q,j}(x^{(j)}) \right) - \tilde{\Phi}_{q,j} \left( \sum_{j=1}^d \tilde{\phi}_{q,j}(x^{(j)}) \right) \right \rVert_{L^\infty([0,1]^d)}. 
\end{align*}
The argumentation from the third step finishes the estimation.
\end{enumerate}
The set of centers required for the approximation is finally given by the union of the centers within the steps above, that is 
\begin{align*}
\bigcup_{q=0}^{2d} \left( \{u_i^{(q,0)}\}_{i=1}^{N^{(q)}} \cup \{ v_i^{(q, 0)} \}_{i=1}^{M^{(q)}} \right).
\end{align*}
The realization of the linear combinations is possible based on \Cref{prop:linear_layer_2}.
\end{delayedproof}

\begin{lemma}
\label{lem:decomposition_function}
Consider a continuous function $h \in C([a,b]), a < b \in \R$ and two points $z_1 < z_2 \in [a, b]$. Then it is possible to construct a decomposition $h = h_1 + h_2$ such that $h_1: \R \rightarrow \R$ and $h_2: \R \rightarrow \R$ have the following symmetry property:
\begin{align*}
h_1(z_1+x) = h_1(z_1-x), ~ h_2(z_2+x) = h_2(z_2-x), \quad \forall x \geq 0.
\end{align*}
\end{lemma}

\begin{proof}[Sketch of proof of Lemma \ref{lem:decomposition_function}]
We give a constructive proof and refer to Figure \ref{fig:proof_visualization_func_decomposition}: Define $d := |z_2 - z_1| > 0$. Subdivide the interval $I := [a,b]$ into pieces of length up to $2d$, starting at $z_1$: 
$I_l = [z_1+2ld, z_1+2(l+1)d]$ and $J_l =[z_1-2(l+1)d, z_1-2ld]$ for $l=0, 1, ...$ .
Observe that $I_l, J_l$ are reflected w.r.t.\ $z_1$. Then it holds $I = \left( \bigcup_{l \geq 0} I_l \cup \bigcup_{l \geq 0} J_l \right) \cap I$. \\
We define the functions $h_1$ and $h_2$ in an iterative fashion starting on the interval $I_0$ and continuing with the intervals $J_0, I_1, J_1, I_2, ...$:
\begin{enumerate}
\item Start: For $x \in I_0$ set $h_1(x) := h(x)$ and $h_2(x) := 0$.
\item Iteratively: \\ 
For $x \in J_l, l \geq 0$:
\begin{itemize}
\item Define $h_1(x)$ via the symmetry condition $h_1(z_1-x) = h_1(z_1+x)$. 
\item Define $h_2(x) := h(x) - h_1(x)$ to satisfy $h(x) = h_1(x) + h_2(x)$.
\end{itemize}
For $x \in I_l, l \geq 1$:
\begin{itemize}
\item Define $h_2$ via the symmetry condition $h_2(z_2+x) = h_2(z_2-x)$.
\item Define $h_1(x) := h(x) - h_2(x)$ to satisfy $h(x) = h_1(x) + h_2(x)$.
\end{itemize}
\end{enumerate}
By construction it holds $h(x) = h_1(x) + h_2(x)$ for $x \in I$. The continuity of $h_1, h_2$ follows from the continuity of $h$ via induction.

\begin{figure}
\begin{center}
\input{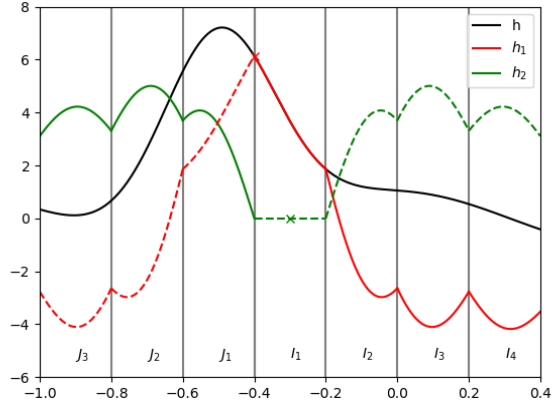}
\caption{Visualization of the decomposition of a function $h: [a,b] \rightarrow \R$ (black) into two symmetric functions $h_1$ (red) and $h_2$ (green), such that it holds $h = h_1 + h_2$.
}  
\label{fig:proof_visualization_func_decomposition}
\end{center}
\end{figure}

\end{proof}

\begin{delayedproof}{th:unbounded_width}
The proof is split into several steps:
\begin{enumerate}
\item Based on a single center, we can approximate any given function that is symmetric with respect to that center on a finite interval: \\ 
Let $\varphi$ be a radial basis function as specified in Assumption \ref{ass:wide_assumptions}. As we only consider a single center for now, say $z_1$, we have 
\begin{align*}
F_2^{(j)}(x) = (f_2 \circ f_1)(x)^{(j)} = \sum_{i=1}^1 \alpha_{2,i}^{(j)} \varphi(|(A_1(x - z_i))^{(j)}|), ~ 1 \leq j \leq d_1 = d_2
\end{align*}
Now we choose the rows of the matrix $A_1$ as positive multiples of each other, i.e.\ $\sigma_j \cdot a$, with
$0 \neq a \in \R^{d_0}, \sigma_j \geq 0, j=1,.., d_1 = d_2$. Then we have
\begin{align*}
F_2^{(j)}(x) &= \alpha_{2,1}^{(j)} \varphi(|\sigma_j a^\top (x - z_1)|), ~ 1 \leq j \leq d_1 = d_2 \\
&= \alpha_{2,1}^{(j)} \varphi(|\sigma_j x_a|), ~~ x_a := a^\top (x-z_1) \in \R,
\end{align*}
which is a scalar valued function in $x_a$. Now choose the last mapping $f_3$ such that a summation is realized, i.e.\ $A_3 = (1, 1, .., 1)^\top$, such that
\begin{align} \label{eq:single_center_approx}
s_1(x_a) := \sum_{j=1}^{d_2} \alpha_{2,1}^{(j)} \varphi(|\sigma_j x_a|).
\end{align}
This expression satisfies $s_1(-x_a) = s_1(x_a)$, thus only consider $x_a \geq 0$:

Due to Assumption \ref{ass:wide_assumptions} it holds that $\Sp \{\varphi(\sigma x_a), \sigma \geq 0 \}$ is dense in $C([0, 2 \sqrt{d} \Vert a \Vert])$. Hence for sufficiently large $d_2$ and suitable coefficients $\alpha_{2,1}^{(j)}$ and $\sigma_j \geq 0$, the sum $s_1(x_a) = \sum_{j=1}^{d_2} \alpha_{2,1}^{(j)} \varphi(|\sigma_j x_a|), x_a \geq 0$ from Equation \eqref{eq:single_center_approx} approximates any function on $C([0, 2\sqrt{d} \Vert a \Vert])$ to arbitrary accuracy:
\begin{align*}
&\forall {\epsilon > 0, f \in C([0, 2\sqrt{d}\Vert a \Vert])} ~~ \exists {d_2 \in \N, \alpha_{2,1}^{(j)}, \sigma_j \in \R_{>0}, j=1, .., d_2} \\ 
&\quad \quad \forall {x_a \in [0, 2\sqrt{d}\Vert a \Vert]} ~~ \bigg| f(x_a) - \underbrace{\sum_{j=1}^{d_2} \alpha_{2,1}^{(j)} \varphi(|\sigma_j x_a|)}_{s_1(x_a)} \bigg| < \epsilon.
\end{align*}
For $x_a \in [-2 \sqrt{d} \Vert a \Vert, 2 \sqrt{d} \Vert a \Vert]$ the function $s_1(x_a)$ corresponds to a symmetric function in $C([-2 \sqrt{d} \Vert a \Vert, 2 \sqrt{d} \Vert a \Vert])$.
\item Leveraging Lemma \ref{lem:decomposition_function}, it is possible to approximate any given function that is univariate in direction $a \in \R^{d_0}$ on a finite interval: \\
First of all, according to Lemma \ref{lem:decomposition_function}, decompose $h \in C([-2 \sqrt{d} \Vert a \Vert, 2 \sqrt{d} \Vert a \Vert])$ into $h = h_1 + h_2$,
with $h_i$ satisfying $h_i(a^\top z_i + y) = h_i(a^\top z_i - y)$ for $y \geq 0, i=1, 2$. 
Here we assume for now that both symmetry centers are different, i.e.\ $a^\top z_1 \neq a^\top z_2 \Leftrightarrow a^\top (z_1 - z_2) \neq 0 \Leftrightarrow a \not\perp z_1-z_2$ is required. This requirement will be addressed in the next step. We have:
\begin{align*}
h(a^\top x) &= h_1(a^\top x) + h_2(a^\top x) \\
&= h_1(a^\top z_1 + a^\top (x-z_1)) + h_2(a^\top z_2 + a^\top (x-z_2)) \\
&=: \tilde{h}_1(a^\top (x-z_1)) + \tilde{h}_2(a^\top (x-z_2))
\end{align*}
Both $\tilde{h}^{(1)}$ and $\tilde{h}^{(2)}$ are symmetric in their input with respect to zero. Therefore according to the first step, there exist $s_1, s_2$ such that
\begin{align*}
\forall y \in [-2\sqrt{d} \Vert a \Vert, 2\sqrt{d} \Vert a \Vert] \qquad 
| \tilde{h}_1(y) - s_1(y) | <& \frac{\epsilon}{2} \\
\text{and} ~ | \tilde{h}_2(y) - s_2(y) | <& \frac{\epsilon}{2}.
\end{align*}
Thus it is possible to estimate:
\begin{align*}
& ~ |h(a^\top x) - s_1(a^\top (x-z_1)) - s_2(a^\top (x-z_2))| \\
\leq& ~ |\tilde{h}_1(a^\top (x-z_1)) - s_1(a^\top (x-z_1))| + |\tilde{h}_2(a^\top (x-z_2)) - s_2(a^\top (x-z_2))| \\
\leq& ~ \frac{\epsilon}{2} + \frac{\epsilon}{2} = \epsilon
\end{align*}
which holds for $x \in [0, 1]^d$, as $|a^\top (x-z_i)| \leq 2 \sqrt{d} \Vert a \Vert, ~ i=1,2$.
\item Finally the approximation of arbitrary univariate functions from the second step is extended to the multivariate input case. For this, a Theorem from Vostrecov and Kreines is used \cite[Theorem 3.2]{pinkus1999approximation}:
\begin{theorem}[Special case of Vostrecov and Kreines, 1961] \label{th:vostrecov_kreines}
The space $\Sp \{ g(a^\top x) ~ | ~ g \in C(\R), a \in \mathcal{A} \}$ for a given $\mathcal{A} \subset \R^d$ is dense in $C(\R^d)$, if $\mathcal{A}$ contains a (relatively) open subset of the unit sphere $\mathbb{S}^{d-1} \equiv \{ y ~ | ~ \Vert y \Vert = 1 \}$.
\end{theorem}
We pick $\mathcal{A} := \{ y \in \mathbb{S}^{d-1} ~ | ~ \Vert y - \frac{z_1-z_2}{\Vert z_1-z_2 \Vert} \Vert < \frac{1}{2} \} \subset \mathbb{S}^{d-1}$ which is a relatively open subset. This choice of $\mathcal{A}$ yields 
\begin{align*}
\frac{1}{4} > \left \lVert y - \frac{z_1 - z_2}{\Vert z_1 - z_2 \Vert} \right \rVert^2 &= \Vert y \Vert^2 + \left \lVert \frac{z_1 - z_2}{\Vert z_1 - z_2 \Vert} \right \rVert^2 - 2 \left\langle y, \frac{z_1 - z_2}{\Vert z_1 - z_2 \Vert} \right\rangle \\
\Leftrightarrow ~~ \left\langle y, \frac{z_1 - z_2}{\Vert z_1 - z_2 \Vert} \right\rangle &> 1 - \frac{1}{8} = \frac{7}{8} > 0,
\end{align*}
i.e.\ $a^\top z_1 \neq a^\top z_2$ for any $a \in \mathcal{A}$. Thus, the requirement $a^\top z_1 \neq a^\top z_2$ of the last step is satisfied. \\
Consider $f \in C([0,1]^d)$ arbitrary. For any $\epsilon > 0$, Theorem \ref{th:vostrecov_kreines} gives the existence of $N \in \N, a_i \in \mathcal{A}, g_i \in C(\R), i=1, .., N$ such that
\begin{align*}
\left| f(x) - \sum_{i=1}^N \alpha_i g_i(a_i^\top x) \right| < \frac{\epsilon}{2}.
\end{align*}
Step $2$ of the proof guarantees arbitrarily accurate approximation of those $\alpha_i g_i(a_i^\top x)$ for $x \in [0,1]^d$, i.e.\ $| a_i^\top x | \leq 2 \sqrt{d} \Vert a_i \Vert$. This is possible, as $a_i^\top z_1 \neq a_i^\top z_2$, 
which is ensured by the choice of $\mathcal{A}$. 
\end{enumerate}
The realization of the linear combinations is possible based on \Cref{prop:linear_layer_2}.
\end{delayedproof}

\end{document}